\documentclass[a4paper,UKenglish,cleveref, autoref, thm-restate]{lipics-v2021}

\hideLIPIcs  %
\nolinenumbers

\usepackage[T1]{fontenc}
\usepackage{graphicx}
\usepackage{hyperref}
\usepackage{color}

\usepackage{newfile}
\usepackage{refcount}
\usepackage{cite}
\usepackage{complexity}
\usepackage{framed}
\usepackage{amsfonts}
\usepackage{amsmath}
\usepackage{amsthm} 
\usepackage{thmtools}
\usepackage{amssymb}
\usepackage{mathtools}
\usepackage{thm-restate}
\usepackage{color}
\usepackage{graphicx}
\usepackage{subcaption}
\usepackage{url}
\usepackage{relsize}
\usepackage{graphicx}
\usepackage{algorithm}
\usepackage[noend]{algpseudocode}
\usepackage{tikz}
\usetikzlibrary{automata, positioning, arrows.meta}
\usepackage{pgfplots}
\usepackage{multirow}
\usepackage{stackengine,scalerel}
\usepackage[strings]{underscore}
\usepackage{paralist}
\usepackage{booktabs}
\usepackage{scalerel}
\usepackage{csquotes}
\usepackage[capitalise]{cleveref}
\usepackage{wrapfig}
\usepackage{xspace}

\crefname{equation}{Eq.}{Eq.}
\crefname{figure}{Fig.}{Fig.}
\crefname{lemma}{Lem.}{Lem.}
\crefname{theorem}{Thm.}{Thm.}
\crefname{corollary}{Cor.}{Cor.}
\crefname{definition}{Def.}{Def.}
\crefname{claim}{Claim}{Claim}
\crefname{section}{Sec.}{Sec.}
\crefname{figure}{Fig.}{Fig.}
\crefname{table}{Tab.}{Tab.}
\crefname{algorithm}{Alg.}{Alg.}
\crefname{observation}{Obs.}{Obs.}
\crefname{example}{Ex.}{Ex.}
\crefname{remark}{Rem.}{Rem.}
\crefname{appendix}{App.}{App.}
\crefname{subfigure}{Fig.}{Fig.}
\usepackage{enumitem}
\usepackage{tikz}  %
\usetikzlibrary{arrows.meta, positioning,arrows, automata, shapes}  %
\usepackage{xcolor}  %
\tikzset{auto, >= stealth}
\tikzset{every edge/.append style={thick, shorten >= 1pt, ->}}
\tikzset{initial/.style={draw, thick, <-, shorten <=1pt}}
\tikzset{player0/.style = {draw, thick, shape=circle, minimum size=0.5cm},node distance=1.5cm}
\tikzset{player1/.style = {draw, thick, shape=rectangle, minimum size=0.7cm},node distance=1.5cm}

\definecolor{customdarkgreen}{rgb}{0.0, 0.5, 0.0} %

\newcommand{\mem}{\mathbb{M}}

\newcommand{\abs}[1]{\left\lvert{#1}\right\rvert}

\renewcommand{\R}{\mathbb{R}}

\newcommand{\qinit}{q_0}
\newcommand{\dist}[1]{\mathcal{D}(#1)}
\newcommand{\pr}{\mathit{pr}}
\renewcommand{\exp}{\mathbb{E}}
\newcommand{\cost}{\mathtt{cost}}
\newcommand{\distance}{\mathtt{D_{TV}}}
\newcommand{\freq}{\mathtt{freq}}

\usepackage[]{todonotes}

\makeatletter
\newsavebox{\@brx}
\newcommand{\llangle}[1][]{\savebox{\@brx}{\(\m@th{#1\langle}\)}%
	\mathopen{\copy\@brx\kern-0.5\wd\@brx\usebox{\@brx}}}
\newcommand{\rrangle}[1][]{\savebox{\@brx}{\(\m@th{#1\rangle}\)}%
	\mathclose{\copy\@brx\kern-0.5\wd\@brx\usebox{\@brx}}}

\makeatother

\renewcommand{\inf}{\mathtt{Inf}}

\newcommand{\src}{\mathit{src}}
\newcommand{\template}{\ensuremath{\Gamma}}

\newcommand{\livegroup}{H_\ell}
\newcommand{\livegroupSingle}{H}
\newcommand{\livegroupSingleN}{H}
\newcommand{\colivegroup}{D}
\newcommand{\safegroup}{S}

\newcommand{\Qavg}{Q_\triangle}
\newcommand{\Qsys}{Q_\bigcirc}
\newcommand{\Qenv}{Q_\square}
\newcommand{\sys}{\bigcirc}
\newcommand{\env}{\square}
\newcommand{\avg}{\triangle}
\newcommand{\shield}[1]{#1'}

\newcommand{\parityTemp}{\textsc{ParityTemplate}}

\newcommand{\thalf}{\frac{1}{2}}
\newcommand{\player}{\diamondsuit}

\newcommand{\buchi}{\ifmmode B\ddot{u}chi \else B\"uchi \fi}
\newcommand{\cobuchi}{\ifmmode co\text{-}B\ddot{u}chi \else co-B\"uchi \fi}
\newcommand{\buchiRegion}{\ensuremath{\mathcal{B}}\xspace}

\newcommand{\sure}{\bullet}
\newcommand{\asure}{\circ}
\newcommand{\psure}{\star}
\newcommand{\win}{\mathcal{W}}

\newcommand{\prob}[1]{\Pr\left(#1\right)}
\newcommand{\distr}{\mu}
\newcommand{\supp}{\mathit{supp}}
\newcommand{\run}{\rho}
\newcommand{\runs}{\text{Runs}}
\newcommand{\fruns}{\text{FRuns}}
\newcommand{\avgreward}[2]{\text{Avg}^{#1}_{#2}}
\newcommand{\discreward}[3]{\text{Disc}^{#2}_{#3}(#1)}
\newcommand{\policy}{\sigma}
\newcommand{\strat}{\policy}
\newcommand{\Policies}{\Sigma}

\newcommand{\hist}{\kappa}
\newcommand{\edge}[2]{(#1,#2)}
\newcommand{\pref}{\textit{pref}}
\newcommand{\last}{\textit{last}}
\newcommand{\normalize}[1]{\mathcal{N}(#1)}
\newcommand{\given}{~\Big\vert~}
\newcommand{\biasact}{B^\policy}
\newcommand{\biasstate}{V^\policy}
\newcommand{\Apol}{\Pi}

\newcommand{\counter}[1]{\mathit{count}_{#1}}
\newcommand{\maxpref}[2]{\mathit{maxpref}^{#1}_{#2}}
\newcommand{\thres}{{\theta}}

\newcommand{\applyStars}{\textsc{ApplySTARs}\xspace}
\newcommand{\applyNaive}{\textsc{ApplyNaive}\xspace}
\newcommand{\benchmark}{\textsc{FactoryBot}\xspace}
\newcommand{\tool}{\texttt{MARG}\xspace}

\newcommand{\Star}{STAR\xspace}
\newcommand{\Stars}{STARs\xspace}
\newcommand{\far}{\texttt{Far}\xspace}
\newcommand{\close}{\texttt{Close}\xspace}
\newcommand{\payoffRegion}{\ensuremath{\mathcal{R}}\xspace}

\definecolor{dkcyan}{rgb}{0.1, 0.3, 0.3}
\definecolor{dkgreen}{rgb}{0,0.3,0}
\definecolor{olive}{rgb}{0.5, 0.5, 0.0}
\definecolor{dkblue}{rgb}{0,0.1,0.5}

\lstdefinelanguage{custom-lang}{
	keywords={let, in, match, with, when, if, then, else, elif, for, to, do, rec, return, new, not, and, while, Input, :},
	keywordstyle=[1]\color{black}\bfseries,
	morekeywords=[2]{append, Set, Dict, Queue, pop, push, add, contains},
	keywordstyle=[2]\sffamily,
	morekeywords=[3]{generalController, prophecy, existentialProjection, composition, isMember, stepComposition, modifiedPlant, h},
	keywordstyle=[3]\color{dkcyan}\ttfamily,
	comment=[l][\color{comment-color}]{//},
	literate=%
	{=}{{{\color{operator-color}=}}}1
	{<-}{{{\color{operator-color}$\leftarrow$}}}1
	{|}{{{\color{dkblue}$\mid$}}}1
	{:}{{{\color{dkblue}:}}}1
	{:=}{{{\color{dkblue}:=}}}1
	{@}{ }1
}
\lstdefinestyle{default}{
	backgroundcolor=\color{white},
	escapeinside={(*}{*)},
	basicstyle=,
	columns=fullflexible,
	commentstyle=\sffamily\color{black!50!white},
	framexleftmargin=1em,
	framexrightmargin=1ex,
	keepspaces=true,
	keywordstyle=\color{dkblue},
	mathescape,
	numbers=left,
	numberblanklines=false,
	numbersep=0.5em,
	numberstyle=\relscale{0.75}\color{gray}\ttfamily,
	showstringspaces=true,
	stepnumber=1,
	xleftmargin=1.2em,
}

\lstnewenvironment{mycode}[1][]
{\small
	\lstset{
		style=default, 
		language=custom-lang,
		#1
	}
}
{}

\graphicspath{{./figures/}{./img/svg/}}
	
\title{Follow the \textsc{STARs}:}
\subtitle{Dynamic $\omega$-Regular Shielding of Learned Probabilistic Policies}
\titlerunning{Dynamic $\omega$-Regular Shielding of Learned Probabilistic Policies}

\author{Ashwani Anand}{Max Planck Institute for Software Systems, Kaiserslautern, Germany}{ashwani@mpi-sws.org}{https://orcid.org/0000-0002-9462-8514}{}
\author{Satya Prakash Nayak}{Max Planck Institute for Software Systems, Kaiserslautern, Germany}{sanayak@mpi-sws.org}{https://orcid.org/0000-0002-4407-8681}{}
\author{Ritam Raha}{Max Planck Institute for Software Systems, Kaiserslautern, Germany}{rraha@mpi-sws.org}{https://orcid.org/0000-0003-1467-1182}{}
\author{Anne-Kathrin Schmuck}{Max Planck Institute for Software Systems, Kaiserslautern, Germany}{akschmuck@mpi-sws.org}{https://orcid.org/0000-0003-2801-639X}{}

\authorrunning{Anand et al.}



\relatedversion{} %

\EventEditors{John Q. Open and Joan R. Access}
\EventNoEds{2}
\EventLongTitle{42nd Conference on Very Important Topics (CVIT 2016)}
\EventShortTitle{CVIT 2016}
\EventAcronym{CVIT}
\EventYear{2016}
\EventDate{December 24--27, 2016}
\EventLocation{Little Whinging, United Kingdom}
\EventLogo{}
\SeriesVolume{42}
\ArticleNo{23}

\begin{document}
	
\label{beginningofdocument}
\newoutputstream{docstatus}
\openoutputfile{main.ds}{docstatus}
\addtostream{docstatus}{Submitted}
\closeoutputstream{docstatus}

\maketitle              
\begin{abstract}
    This paper presents a novel dynamic post-shielding framework that enforces the full class of $\omega$-regular correctness properties over pre-computed probabilistic policies. This constitutes a paradigm shift from the predominant setting of safety-shielding -- i.e., ensuring that nothing bad ever happens -- to a shielding process that additionally enforces liveness -- i.e., ensures that something good eventually happens. At the core, our method uses \emph{Strategy-Template-based Adaptive Runtime Shields (STARs)}, which leverage permissive strategy templates to enable post-shielding with minimal interference. As its main feature, STARs introduce a mechanism to \emph{dynamically control interference}, allowing a tunable enforcement parameter to balance formal obligations and task-specific behavior \emph{at runtime}. This allows to trigger more aggressive enforcement when needed while allowing for optimized policy choices otherwise.
    In addition, STARs support runtime adaptation to changing specifications or actuator failures, making them especially suited for cyber-physical applications. 
    We evaluate STARs on a mobile robot benchmark to demonstrate their controllable interference when enforcing (incrementally updated) $\omega$-regular correctness properties over learned probabilistic policies. 
\end{abstract}

\sloppy
\section{Introduction}
\label{sec:introduction}

\noindent\textbf{Motivation.}
Adhering to formal correctness while simultaneously optimizing performance is a core challenge in the design of autonomous cyber-physical systems (CPS)~\cite{seshia2022toward,dalrymple2024towards}. While machine learning techniques -- especially reinforcement learning -- are highly effective at generating policies optimizing quantitative rewards, they often struggle to enforce qualitative correctness criteria such as safety or liveness. In contrast, formal methods excel at specifying and verifying such qualitative properties, typically using temporal logic or automata. Integrating these strengths has been a growing focus in safe and explainable autonomy and has led to a rich body of work integrating logical specifications into policy synthesis via multi-objective formulations~\cite{ChatterjeeD10,ChatterjeeD11,Helouet2019,Savas2024}, or into policy reinforcement learning through automata-based reward shaping~\cite{Hahn,HahnPSS0W23,KazemiPSS0V22,YangMRR23,Busatto-GastonC21,KretinskyPR18}. While these approaches can yield policies that satisfy complex goals while adhering to formal specifications, they incorporate the specification into the synthesis or learning procedures -- requiring re-training when formal objectives, environment conditions, or reward structures change.

\smallskip
\noindent\textbf{Shielding.}
A promising alternative to retraining is shielding\footnote{While there are pre-shielding frameworks that apply shielding during training, we focus solely on post-shielding at runtime.} -- a lightweight, runtime approach that monitors and, if necessary, overrides the actions proposed by a policy to maintain compliance with a formal specification.
Shields treat existing policies as a black box and ensure correctness in a \emph{minimally interfering} manner i.e., they intervene if \emph{and only if} the systems executions will (surely) violate the specification (in the future). The concept of shielding traces back to the foundational works on runtime monitoring of program executions in computer science~\cite{lee1999runtime,havelund2002synthesizing}, and formal supervision of feedback control software in engineering~\cite{ramadge1989control}. More recently, several shielding frameworks tailored for learned policies in autonomous CPS have been introduced (see surveys by~\cite{Review23SafetyFilterHsuHuFisac,Review23DataDrivenSafetyFiltersAmesTomlinZeilingerEtAl,konighofer2022correct,OdriozolaOlaldeZA23}). Crucially, because shielding operates at runtime, it allows goals, safety regions, or constraints to be updated online, replacing retraining with a re-computation of the shield which often requires significantly less computational resources.

\smallskip
\noindent\textbf{Challenges.}
The primary challenge in shielding is to ensure correctness and minimal interference -- with only black-box runtime access to the (learned) nominal policy. 
Prior frameworks have primarily focused on safety, where synthesizing maximally permissive shields (which are `inherently' minimally interfering) is tractable~\cite{alshiekh2018safe,konighofer2023online,reed2024shielded}. In contrast to safety, shielding for liveness is challenging because such specifications (i) do not easily allow for permissive strategies, and (ii) only manifest themselves in the infinite limit, hardening their enforcement at runtime given only a finite prefix.
However, the need for shields which enforce the full class of \emph{$\omega$-regular specifications} naturally arises in CPS applications. %
A natural example are warehouse robots which should not drop a parcel (safety) but should also 
deliver it eventually (liveness)[see~\url{https://youtu.be/AAhHmquT9Fs}].

\begin{figure*}[t]
	\centering
	\begin{subfigure}[b]{0.21\textwidth}
		\includegraphics[width=\textwidth]{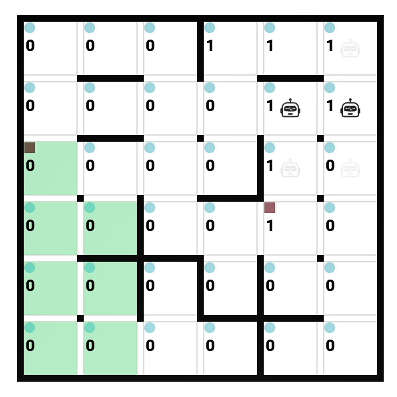}
		\caption{Unshielded Bot}
		\label{screenshot:botunshielded}
	\end{subfigure}
	\begin{subfigure}[b]{0.21\textwidth}
		\includegraphics[width=\textwidth]{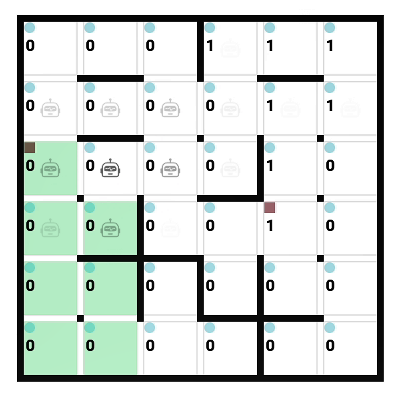}
		\caption{Shielded Bot}
		\label{screenshot:botShielded}
	\end{subfigure}
    \hfill
    \begin{subfigure}[b]{0.25\textwidth}
		\includegraphics[width=\textwidth]{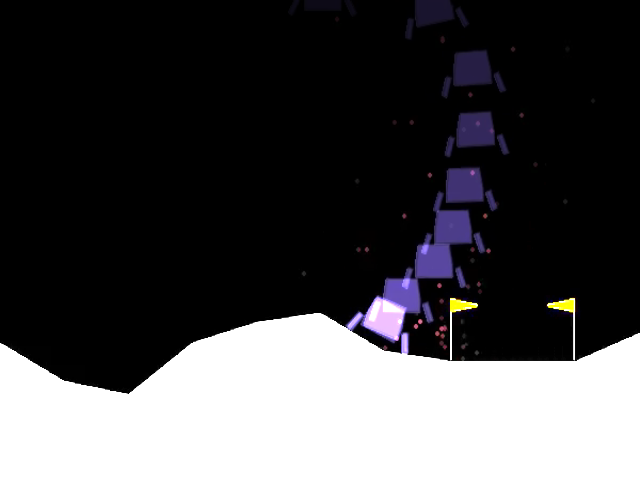}
		\caption{Unshielded Lander}
		\label{screenshot:landerUnshielded}
	\end{subfigure}
    \hspace{0.5em}
    \begin{subfigure}[b]{0.25\textwidth}
		\includegraphics[width=\textwidth]{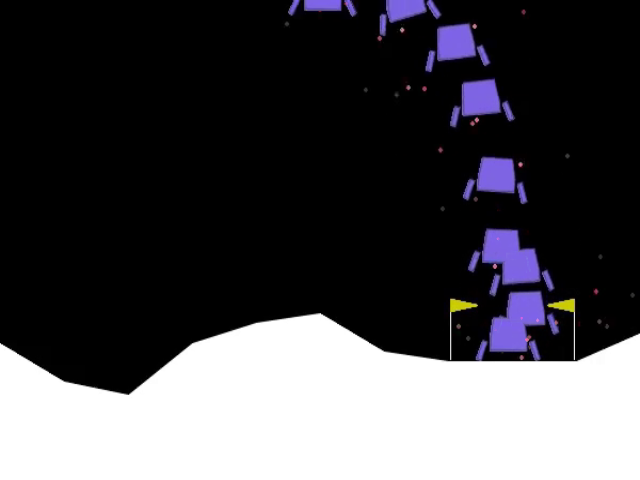}
		\caption{Shielded Lander}
		\label{screenshot:landerShielded}
	\end{subfigure}
	\caption{Applying \Stars{} on different benchmarks. For \benchmark : The agent must visit \buchiRegion (green), while maximizing the average reward (small numbers in cells). Images show the agent's heatmap without a shield (a), with our shield (b). For LunarLander: The agent must land on the helipad (between the flags) while avoiding crashing. Images show the lander's trajectory without a shield (c), with our shield (d).}

	\label{figure:screenshots}
\end{figure*}

\smallskip
\noindent\textbf{Contribution.}
To close this gap, this paper presents a \emph{dynamic} shielding framework -- \emph{Strategy Templates based Adaptive Runtime Shields (\Stars{})} -- which enforces \emph{the full class of $\omega$-regular specifications} over learned probabilistic policies. In addition to being correct and minimal interfering, our shields are equipped with an \emph{enforcement parameter~$\gamma$}.  %
This allows different levels of specification enforcement in different operational contexts at 
runtime: if liveness properties must be satisfied urgently (e.g., reaching a safe zone after an 
alarm) $\gamma$ can be increased while $\gamma$ is kept low otherwise, to exploit 
mission-specific performance objectives optimized during policy learning.

\smallskip
\noindent\textbf{Key Idea.}
The core idea behind the \Stars{} framework is to utilize
\emph{strategy templates} recently introduced by~\cite{AnandNS23} as their main building block. Strategy templates offer an alternative representation of strategies in two-player parity games (resulting from an $\omega$-regular specification)
that condense an infinite number of winning strategies into a simple and efficiently computable data structure. This makes them (i) 
truly permissive, enabling minimal inference shielding, and (ii) localizes required future 
progress (over a known transition graph\footnote{We only need access to the graph \emph{structure}, not to transition probabilities, rewards or computed policies.}), enabling a purely history-induced evaluation of liveness properties. Thereby, strategy templates constitute the \enquote{missing piece} that makes liveness shielding tractable. In addition, strategy templates are inherently robust and composable, allowing STARs to handle dynamic runtime changes, such as unavailable actions or shifting goals.

\subsection{Applications}\label{sec:introLuna}

\textbf{\benchmark} is a novel benchmark, %
which simulates multiple \emph{gridworld environments} in OpenAI Gym~\cite{gym} such as frozen lake, taxi or cliff walking, typically used to evaluate RL policies and represents a simplified version of the snake example 
used to evaluate a dynamic safety shield~\cite{konighofer2023online}. Furthermore, adaptability in such environments, e.g., dynamic goal changes in Frozen Lake, are also subsumed by the more general capabilities of our \benchmark{} setup.
In \benchmark{}, the nominal agent policy optimizes a reward function and \Stars{} are used to guide the agent towards satisfying 
(generalized) Büchi objectives\footnote{The choice of simple objectives is for 
illustration only. Our algorithm can handle the full class of $\omega$-regular objectives.} which might change at runtime. %

\textbf{Overcooked-AI}~\cite{carroll2019utility} is a widely used benchmark for collaborative reinforcement learning with multiple actors. Here, autonomous agents are trained to repeatedly perform cooking tasks. We use LTL specifications to encode additional recipe requirements of produced dishes. \Stars{} are used to enforce the (additional) production of soups satisfying these requirements infinitely often. %

\textbf{LunarLander}~\cite{brockman2016openai} is a standard benchmark in reinforcement learning. 
The simulated lunar lander can be controlled by using its directional thrusters to navigate the environment.
The classical objective is to train the lander to land on a designated helipad without crashing.
To demonstrate the utility of~\Stars{}, we slightly modified this benchmark to spawn the helipad at a \emph{random position} at runtime and trained a baseline policy that just ensures that the lander touches down safely
--- though not necessarily on the helipad. We then apply \Stars{} to enforce landing on the pad \emph{at runtime} while landing safely.

\smallskip
\noindent\textbf{Liveness Shielding.} 
Both \benchmark and Overcooked-AI are based on a known finite MDP. We can therefore synthesize shields over a game graph derived from this MDP which ensures that the agent \emph{always} fulfills the shielded objective (i.e.\ visiting the green region or producing particular soups always again) with a frequency determined by the enforcement parameter $\gamma$. On the other hand, LunarLander is based on an eight-dimensional hidden continuous MDP, which prevents us from  
computing a provably correct shield. %
Instead, we construct a very simple 
2-dimensional grid-abstraction of the landers x/y position that naively assumes that each control action makes the lander move one grid cell down, right, left, or up, respectively. Even 
though the actual dynamics of the lander are very complex, leading to quite different abstract 
behavior, this simple graph allows us to compute a shield which ensures that the lander lands fast and
safe on the helipad in $87\%$  of $200$ instances we have simulated with the right choice of $\gamma$. 
This demonstrates that \Stars enable the blending of quantitative objectives optimized in RL and qualitative liveness objectives at runtime, with promising future applications of \Stars for deep RL shielding.

\smallskip
\noindent\textbf{Supplements.}
We report details on the experiments in \Cref{sec:experiments}.
Recordings of different instances of the outlined experiments are available at {\scriptsize\textcolor{blue}{\url{http://anonymous.4open.science/w/MARGA}}}.

\subsection{Related Work}
(Post-) shielding approaches have so far focused mainly on \emph{safety} shielding, where our work is the closest related to the works of~\cite{alshiekh2018safe,konighofer2023online,reed2024shielded}.
Similar ideas are also used for \emph{policy repair} w.r.t.\ safety violations by~\cite{pathak2018verification,fuzzrepair,Waveren2022,Zhou2020} or via (partial) re-synthesis  by~\cite{KantarosPappasEtAllRobots21,PacheckKressGazitRobots23}. In contrast, \Stars{} directly inherent robustness and adaptability properties of strategy templates that typically circumvent the need for strategy repair and achieve necessary strategy adaptations directly via shielding.
For general $\omega$-regular specifications, our work is closely related to the runtime optimization~\cite{chatterjee2023shielding} which propose a similar \emph{blending} of a nominal policy with an additional liveness objective, however, via a very different shield synthesis technique. In contrast to our work, the shield synthesized in~\cite{chatterjee2023shielding} uses a fixed enforcement parameter, does not exploit probabilistic policies, and allows no dynamic adaptation in the specification or in the graph. 

This work also relates to approaches synthesizing policies for mean-payoff objectives under $\omega$-regular constraints~\cite{AlmagorKV16,ZeitlerMMS024}, and to pre-shielding frameworks~\cite{KazemiPSS0V22,Hahn,YangMRR23,OdriozolaOlaldeZA23,Busatto-GastonC21,KretinskyPR18}. Achieving similar optimality in post-shielding is harder, but \Stars{} can approach optimal rewards when the winning region is strongly connected. Addressing correctness without full MDP access remains a challenge for future work.

Other work that is not directly related to shielding, but is relevant to our work
along with all proofs and additional formalizations can be found in the appendix. 

\section{Preliminaries}
\label{sec:prelims}
This section provides an overview of notation and concepts.

\smallskip
\noindent\textbf{Notation.}
We denote by $\R$ the set of real numbers and $ [a;b]$ represents the interval $ \{a, a+1 \cdots , b\} $. 
We write $\Policies^*$ and $\Policies^\omega$ to denote the set of finite and infinite sequences of elements from a set $\Policies$, respectively.
A \emph{probability distribution} over a finite set $S$ is denoted as a function $\distr: S \mapsto [0,1]$ such that $\sum_{s\in S} \distr(s) =1$. The set of all probability distributions over $S$ is denoted as $\dist{S}$. The \emph{support} of a distribution $\distr$ is the set $\supp(\distr) = \{s \in S \mid \distr(s)>0\}$.
Given two distributions $\distr_1,\distr_2\in\dist{S}$, the \emph{total variation distance} between $\distr_1$ and $\distr_2$ is defined as: 
$\distance(\distr_1,\distr_2) = \frac{1}{2}\sum_{s\in S} \abs{\distr_1(s) - \distr_2(s)}$.
Given any function $\distr: S \mapsto \R$, we use $\normalize{\distr}\in\dist{S}$ to denote its normalized distribution: $\normalize{\distr}(s) = \frac{\distr'(s)}{\sum_{s'\in S} \distr'(s')},\quad\text{where }\distr'(s') = \max(\distr(s'),0).$

\smallskip
\noindent\textbf{Markov Decision Process.}
A \emph{Markov Decision Process} (MDP) is a tuple $M = \langle Q, A, \Delta, \qinit\rangle$ where $Q$ is a finite set of states, $A$ is a finite set of actions, $\Delta: Q \times A \mapsto \dist{Q}$ is a (partial) transition function and $\qinit \in Q$ is the initial state. For any state $q \in Q$, we let $A(q)$ denote the set of actions that can be selected in state $q$. 
A \emph{strongly connected component (SCC)} of an MDP $M$ is a maximal set of states $Q' \subseteq Q$ such that for every pair of states $q,q' \in Q'$, there exists a path from $q$ to $q'$ in $Q'$ with non-zero probability.

Given a state $q$ and an action $a \in A(q)$, we denote the probability of reaching the successor state $q'$ from $q$ by taking action $a$ as $\pr(q'|q,a)$. A run $\rho$ of an MDP $M$ is an infinite sequence in $Q \times (A \times Q)^\omega$ of the form $\qinit a_0 q_1 \ldots$ such that $\pr(q_{i+1}|q_i,a_i) > 0$. A \emph{finite run} of \emph{length} $n$ is a finite such sequence $\hist = q_0 a_0 \ldots q_n a_n$ or $\hist = q_0a_0\ldots q_n$. 
We write $\run[i]$ to denote the $i^{th}$ state-action pair $(q_i,a_i)$ appearing in $\rho$, $\run[i;j]$ to denote the infix $q_i a_i \ldots q_j$ for $j \geq i$, and $\run[j;\infty]$ to denote the suffix $q_j a_j \ldots$.
These notations extend to the case of finite runs analogously. We write $\runs^M$ (resp.\ $\fruns^M$) to denote the set of all infinite (resp.\ finite) runs of $M$. We denote the last state of a finite run $\rho$ as $\last(\rho)$.

A \emph{policy} (or, strategy) in an MDP $M$ is a function $\policy: \fruns^M \mapsto \dist{A}$ such that $\supp(\policy(\rho)) \subseteq A(\last(\rho))$. Intuitively, a policy maps a finite run to a distribution over the set of available actions from the last state of that run.  
A policy is \emph{stochastic} if $|\supp(\policy(\hist q))|= |A(q)|$ for every history $\hist q$.
A run $\rho = \qinit a_0 q_1\ldots$ is a $\policy$-run
if $a_i \in \supp(\policy(\qinit a_0 \ldots q_i))$. 
Given a measurable set of runs $P \subseteq \runs^M$ or finite runs $P \subseteq \fruns^M$, $\Pr_{\policy}[P]$ is the probability that a $\policy$-run belongs to $P$. We use $\runs^{M^\policy}$ to denote the set of all $\policy$-runs and by $\Apol_M$ the set of all policies over $M$.

\smallskip
\noindent\textbf{$\omega$-Regular Objectives and (Almost) Sure Winning.}
Given an MDP $M$, an \emph{objective} is defined as a set of %
runs $\Phi\subseteq \runs^M$. %
An $\omega$-regular objective can be canonically represented by a \emph{parity objective} (possibly with a larger set of states~\cite{BaierKatoen08}) $\Phi= \textsc{Parity}[c]$ which is defined using a coloring function $c: Q \rightarrow [0;d]$ that assigns each state a \emph{color}. The parity objective $\textsc{Parity}[c]$ contains all runs $\rho \in \runs^M$ for which the highest color (as assigned by the coloring function $c$) appearing infinitely often is even. 
To define them formally, let us write $\inf_Q(\rho)$ (resp.\ $\inf_{Q\times A}(\rho)$) denoting the set of all states (resp. state-action pairs) which occur infinitely often along $\rho$.
Then, the parity objective is defined as
$\textsc{Parity}[c] = \{\rho \in \runs^M \mid \max \{c(q)\mid q \in \inf_Q(\rho)\} \text{ is even}\},$
The parity objective $\textsc{Parity}[c]$ reduces to a \emph{B\"uchi} objective, if the domain of $c$ is restricted to two colors $\{1,2\}$.

Given an MDP $M$ and an objective $\Phi \subseteq \runs^M$,
a run is said to \emph{satisfy} $\Phi$ if it belongs to $\Phi$.
A policy $\policy$ is said to be \emph{surely} (resp. \emph{almost surely}) \emph{winning} from a state $q$, if in the MDP $M^q = \langle Q, A, \Delta, q\rangle$,
every $\policy$-run satisfies $\Phi$ (resp. $\Pr_\policy[\Phi] = 1$).
We collect all such states from which a surely (resp. almost surely) winning policy exists in the winning region $\win^{\sure}_\Phi$ (resp. $\win^{\asure}_\Phi$).
We say a policy $\policy$ is \emph{surely} (resp. \emph{almost surely}) \emph{winning} in MDP $M$ for objective $\Phi$,
denoted by $(M,\policy) \models_\sure \Phi$ (resp. $(M,\policy) \models_\asure \template$),  if it is surely (resp. almost surely) winning from every state in winning region.

Standard algorithms for computing winning policies for sure (resp. almost sure) parity objectives in MDPs reduce the problem to 2-player (resp. $1\thalf$-player) parity games~\cite{AlmagorKV16}. This is obtained by partitioning the states of the MDP into two sets: the system player controls the states in one set and chooses the next action; the environment/random player controls the states in the other set and chooses the distribution over the next states.

\smallskip
\noindent\textbf{Strategy Templates.}\label{sec:PeSTel}
\emph{Strategy templates}~\cite{AnandNS23} collect an infinite number of strategies over a (stochastic) game in a concise data structure.
With the standard reduction of MDPs with parity objectives to 2-player (or $1\thalf$-player) parity games, we can use strategy templates to represent winning policies in MDPs.

Given an MDP $M = \langle Q, A, \Delta, \qinit\rangle$ and a strategy template is a tuple $\template= (\safegroup,\colivegroup,\livegroup)$ comprising a set of \emph{unsafe} state-action pairs $\safegroup\subseteq Q \times A$, a set of \emph{co-live} state-action pairs $\colivegroup\subseteq Q \times A$, and a set of \emph{live-groups} $\livegroup \subseteq 2^{Q \times A}$.
A strategy template $\template$ over $M$ induces the following set $\runs^\template$ of infinite runs
\begin{align*}
    \small
  \left\{\rho \in \runs^M \Bigg |
 \begin{array}{ll}
  & \forall (q,a) \in \safegroup: qa \not\in \rho\\
  \wedge &\forall (q,a) \in \colivegroup: qa \not\in \inf_{Q \times A}(\rho)\\
  \wedge &\forall H\in \livegroup: \src(H)\cap \inf_Q(\rho) \neq \emptyset 
  \\&\qquad\quad\rightarrow H \cap \inf_{Q \times A}(\rho) \neq \emptyset
 \end{array}
 \right\}
\end{align*}
Intuitively, a run $\rho\in \runs^\template$  satisfies the following conditions: (i) $\rho$ never uses the unsafe actions in~$\safegroup$,
(ii)  $\rho$ stops using the co-live actions in~$\colivegroup$ eventually, and
(iii) if $\rho$ visits the set of source states of a live-group $\livegroupSingleN \in \livegroup$ infinitely often, then it also uses the actions in $\livegroupSingleN$ infinitely many times. 
Given an MDP $M$, a policy $\policy$ in $M$ \emph{follows} a template $\template$ if
$(M,\policy) \models_\sure\runs^\template$. If $M$ is clear from the context we often abuse notation and write $\policy\models_\sure\template$ if $\policy$ follows $\template$.
A strategy template $\template$ is said to be \emph{surely} (resp.\ \emph{almost surely}) winning for the parity objective $\Phi$ if every policy that follows $\template$ is  surely (resp.\ almost surely) winning for $\Phi$.

Existing algorithms~\cite{AnandNS23,phalakarn2024winningstrategytemplatesstochastic} can compute a \emph{surely} (resp.\ \emph{almost surely}) \emph{winning} strategy template $\template_\sure$ (resp. $\template_\asure$) for a (stochastic) parity game.
Using these algorithms along with the standard reduction of MDPs to 2-player (or $1\thalf$-player) parity games, one can compute a winning strategy template for the MDP.
We denote the algorithm that computes a winning strategy template by $\parityTemp_\psure(M,\Phi)$ with $\psure\in\{\sure,\asure\}$.

\section{Dynamic $\omega$-Regular Shielding}
\label{sec:shieldalgo}\label{sec:shielding}
This section formalizes our novel dynamic shielding framework via Strategy-Template-based-Adaptive-Runtime-Shields (\Stars{}) schematically depicted in~\Cref{fig:overview}. 

\begin{figure*}[t]
 \begin{center}
	\scalebox{1}{
  \includegraphics[width=0.55\textwidth]{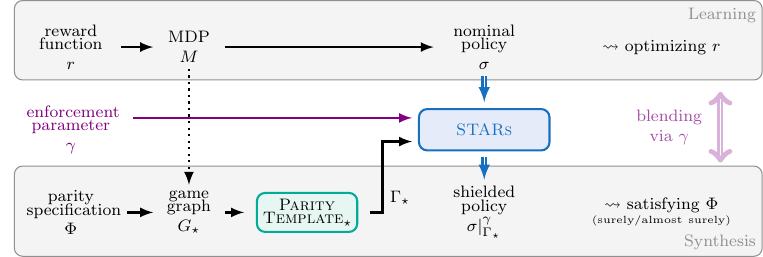}\hspace{0.4cm}
  \includegraphics[width=0.4\textwidth]{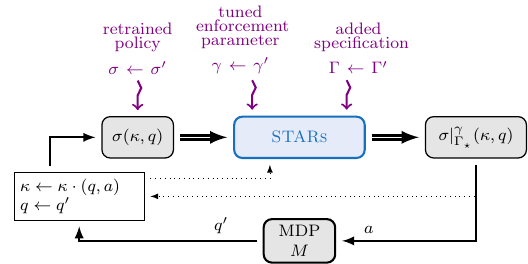} }
  \caption{Overview of \Stars synthesis (left) and runtime-application of \Stars (right). The detailed operation of \Stars is illustrated in \cref{fig:stesh}.
Cyan components are taken from the literature. Purple components illustrate dynamic adaptability.
  }\label{fig:overview}
 \end{center}
\end{figure*}

\begin{figure}[t]
 \begin{center}
  \includegraphics[width=0.5\textwidth]{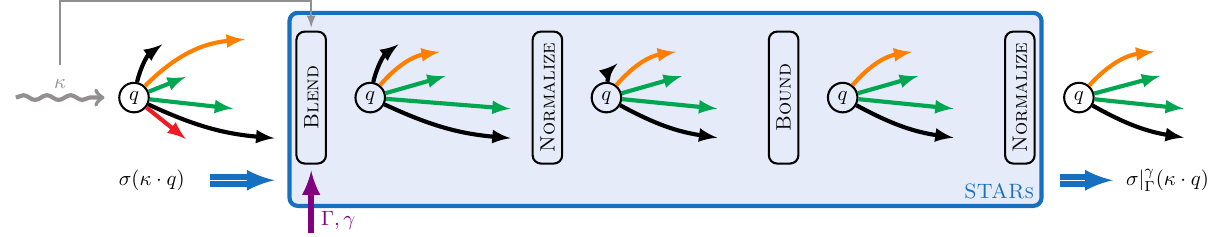}
  \caption{Illustration of dynamic interference via \Stars{}. Length of arrows indicate the relative probability of the corresponding action in $\mu\in\dist{A(q)}$. The strategy template $\Gamma=(\safegroup,\colivegroup,\livegroup)$ is illustrated via colors red ($\safegroup$), orange ($\colivegroup$) and green ($\livegroup$). Blending applies~\eqref{eq:shieldpolicy:gamma} in \cref{def:stesh}, bounding applies~\eqref{eq:shieldpolicy:eps} in \cref{def:stesh} and normalizing applies standard normalization, respectively.  }\label{fig:stesh}
 \end{center}
\end{figure}

\subsection{Dynamical Interference via \Stars{}}\label{sec:Applystesh}
Given a winning strategy template $\template_\psure = (\safegroup,\colivegroup,\livegroup)$ for the parity objective $\Phi$ interpret over an MDP, \Stars dynamically \emph{blend} a given nominal policy $\policy$ with the safety and liveness obligations of $\Phi$ localized in $\template_\psure$ as depicted in \cref{fig:stesh}.

Intuitively, to comply with the safety template $\safegroup$, \Stars set the probabilities of unsafe actions to zero, thereby preventing runs to reach states from where $\Phi$ cannot be satisfied. For each edge in the co-live group $\colivegroup$, \Stars maintain a counter that tracks how many times the edge has been sampled. Each time the edge is sampled, its probability is reduced, ensuring that runs eventually avoid co-live edges. 
Similarly, for each live group $\livegroup$, \Stars maintain a counter to track how many times the policy visits the source states of the group without sampling any of its actions. With each such visit, the shield incrementally increases the probability of sampling these actions based on the counter value. This guarantees that eventually, one of the actions in the live group is sampled (with probability close to $1$) if its source states are visited often enough. Once an action is sampled, the corresponding counter is reset, and the process repeats. 
In order to formalize the above intuition, we first formally define the history-dependent counter function for co-live and live groups discussed above.

\begin{definition}\label{def:counters}
 Let $M$ be an MDP and $\template=(\safegroup,\colivegroup,\livegroup)$ a strategy template interpreted over $M$. Further, let $\hist = \qinit a_0 q_1 a_1 \ldots q_n a_n\in FRuns^M$ be a finite run over $M$. Then we define for all $(q,a)\in D$: 
 $\counter{\edge{q}{a}}(\hist) := |\{i \mid \hist[i] = (q,a)\}|$,
 and for all $\livegroupSingle \in \livegroup$: 
 $\counter{\livegroupSingle}(\hist) := |\{i > \maxpref{\hist}{\livegroupSingle} \mid \hist[i] \in \{(q,a)\mid q\in\src(\livegroupSingle)\}\}|$,
 where $\maxpref{\hist}{\livegroupSingle} := j \text{ such that } \hist[j] \in \livegroupSingle \text{ and } \hist[j;\infty] \cap \livegroupSingle = \emptyset$.
\end{definition}

The above definition allows us to formally define how a \Star modifies the probability distribution $\distr(A(q))$ over actions chosen by $\strat$ in the current state $q$ reached with history $\hist$ using a template $\template_\psure$. Intuitively, the \emph{bias} towards satisfying $\Phi$ introduced by the counter-based modification of  $\distr(A(q))$ can be \emph{tuned} by an enforcement parameter $\gamma$, and a threshold parameter $\thres$, which can be changed dynamically at runtime.

\begin{definition}[\Stars{}]\label{def:stesh}
 Fix an MDP $M$, a finite run $\hist\in FRuns^M$ with $q=\last(\hist)$, a strategy template $\template$ interpreted over $M$, an enforcement parameter $\gamma$, and a threshold $\thres$. Then, the probability distribution $\mu\in\dist{A(q)}$ induces a \emph{shielded distribution} $\overline{\mu}\in\dist{A(q)}$ with $\overline{\mu}:=\normalize{\mu'}$ s.t.\ for all $a\in A(q)$ %
 \begin{subequations}\label{eq:shieldpolicy}
 \small
  \noindent
  \begin{equation}\label{eq:shieldpolicy:eps}
    \mu'(a):=\begin{cases}
      0 & \hspace{-0.2cm}\text{if } \normalize{\mu''}(a)\leq\thres,\\
      \normalize{\mu''}(a) & \text{otherwise}.
    \end{cases}
  \end{equation}
  \begin{equation}\label{eq:shieldpolicy:gamma}
    \mu''(a):= 
    \begin{cases}
      0 & \text{if } \edge{q}{a} \in \safegroup,\\
      \mu(a) - \gamma \cdot \counter{\edge{q}{a}}(\hist) & \text{if } \edge{q}{a} \in \colivegroup,\\
      \mu(a) + \gamma \cdot \counter{\livegroupSingle}(\hist) & \text{if } \edge{q}{a} \in \livegroupSingle,  \livegroupSingle \in \livegroup.
    \end{cases}
  \end{equation}
  \end{subequations}
\noindent
We write $\overline{\mu}=\Stars{}(\mu,\hist,\template,\gamma,\thres)$ to denote that $\overline{\mu}$ is obtained from $\mu$ via~\eqref{eq:shieldpolicy}. 
\end{definition}

The effect of shielding a policy as formalized in \cref{def:stesh} is illustrated in \cref{fig:stesh}. 
Intuitively,~\eqref{eq:shieldpolicy:gamma} ensures that the probability of taking certain actions in state $q$ is adapted via the counters induced by the history of the current run and the enforcement parameter $\gamma$. For $\gamma\sim 1$, these updates are very aggressive and for $\gamma\sim 0$, they are very mild. As the resulting function $\mu''$ is not a probability distribution anymore (as probabilities over $A(q)$ do not sum up to 1), 
we normalize it as in~\eqref{eq:shieldpolicy:eps} to impose the threshold $\thres>0$ to make sure that a live edge is surely taken after a finite number of time steps (dependent on $\gamma$ and~$\policy$).
In the end, we normalize the resulting distribution to obtain the final distribution.

\begin{remark}\label{rem:cornercase}
  We note that if there is an unsafe action $a$ in $\safegroup$ such that the current distribution $\mu$ assigns a probability of $1$ to it, i.e., $\mu(a)=1$, then $\overline{\mu}$ as in \cref{def:stesh} will be not well-defined as it assigns zero probability to all actions. 
  This corner case can be handled by perturbing the distribution $\mu$ slightly, e.g., by adding a small $\varepsilon>0$ to all actions before applying \Stars{}. 
\end{remark}

With the definition of shielded distribution in \cref{def:stesh}, the definition of a \emph{shielded policy} immediately follows.

\begin{definition}\label{def:shieldedPolicy}
 Given any MDP $M$, a strategy template $\template$ interpreted over $M$, a threshold $\thres$, and an enforcement parameter $\gamma >0$, a stochastic policy $\policy$ in $M$ induces the shielded policy $\policy|^{\template,\thres}_{\gamma}:FRuns^M \mapsto \dist{A} $ s.t.\ 
 $\policy|^{\template,\thres}_{\gamma}(\kappa)=\Stars{}(\policy(\hist),\hist,\template,\gamma,\thres).$
\end{definition}

To avoid the corner case discussed in \cref{rem:cornercase}, the definition assumes that the initial policy $\policy$ is stochastic, i.e., $\supp(\policy(\hist q)) = A(q)$ for all histories $\hist q$. This is without loss of generality as any deterministic policy can be converted into a stochastic one as discussed in \cref{rem:cornercase}.
Furthermore, as the resulting shielded policy is dependent on the history of a run, a policy is actually shielded via \Stars{} \emph{online} while generating a \emph{shielded run}, as illustrated in \cref{fig:overview}~(right). We emphasize that \Stars{} never modify the underlying policy and thereby maximize modularity between the nominal policy and constraint enforcement.

\subsection{Correctness of \Stars{}}\label{sec:Correctstesh}
Correctness of \Stars{} follows directly from the fact that they are based on  \emph{winning strategy templates} $\template_\psure$, which implies that the shielded policy $\policy|^{\template_\psure,\thres}_{\gamma}$ satisfies the objective $\Phi$ (almost) surely, if it follows the template. It remains to show that the shielded policy indeed follows the template.
As~\eqref{eq:shieldpolicy:gamma} ensures that the shielded policy assigns zero probability to unsafe edges and that the probability of taking co-live edges is reduced with each visit, the shielded policy will never take an unsafe edge and will eventually avoid co-live edges.
Furthermore, as~\eqref{eq:shieldpolicy:eps} increments the counter for live groups each time the source states are visited without any action from the group being taken, the shielded policy will eventually take an action from the live group.
In total, the shielded policy will follow the template $\template_\psure$ and therefore the shielded run satisfies $\Phi$.

\begin{restatable}{theorem}{restateShieldTemplate}
    \label{thm:shieldtemplate}\label{corollary:shieldedPolicy}
    Given the premises of \cref{def:shieldedPolicy} it holds that $\policy|^{\template_\psure,\thres}_{\gamma}$ follows $\template_\psure$.
    Moreover, if $\template_\psure:=\parityTemp_\psure(M,\Phi)$, then, every $\policy|^{\template_\psure,\thres}_{\gamma}$-run from the winning region of $\Phi$ satisfies $\Phi$ surely ($\sure$)/almost surely ($\asure$). %
\end{restatable}

\subsection{Minimal Interference of \Stars{}}\label{sec:MinimalInterference}
Minimal interference of \Stars{} is, unfortunately, less straightforward to formalize. Based on existing notions, we characterise two orthogonal notions: (i) a minimal deviation in the distribution of observed histories, and (ii) a minimal expected average shielding cost measured in the expected number of non-optimal action choices.

\emph{History-based minimal interference} is inspired by a similar notion from~\cite{Elsayed-AlyBAET21} for safety shields: an action must be deactivated after a history $\hist$, if \emph{and only if} there exists a nonzero probability that the safety constraint would be violated in a bounded extension of $\hist$, regardless of the agent's policy. This argument extends to \Stars{} for both safety and co-live templates, ensuring minimal interference in these settings.
However, defining minimal interference for liveness templates is more challenging due to their inherently infinite nature, making bounded violations inapplicable. Instead, we establish minimal interference by showing that for any bounded execution $\hist$ that can be extended to a run that satisfies the liveness template, the probability of observing $\hist$ in the shielded execution remains close to its probability under the nominal policy. 

\begin{restatable}{theorem}{restateMinimalInterference}\label{thm:minimalinterference}
    Given the premises of \cref{corollary:shieldedPolicy} with $\shield{\policy} = \policy|^{\template_\psure,\thres}_{\gamma}$,   
    for any $\varepsilon > 0$ and for all $l \in \mathbb{N}$, there exist  parameters $\gamma,\thres > 0$ s.t.\ for all histories $\hist$ of length $l$ with $\hist \in \pref(\Phi)$, it holds that
    $\Pr_{\shield{\policy}}(\hist) > \Pr_{\policy}(\hist) -  \varepsilon$, where $\pref(\Phi)$ is the set of prefixes of runs satisfying $\Phi$.
\end{restatable}

\noindent\emph{Minimal shielding costs} are inspired by a similar notion in~\cite{chatterjee2023shielding}, where a cost function is used to measure how much a (deterministic) shield changes the action choices of a (pure) policy. 
A natural extension of this notion to stochastic policies is to define a shielding cost based on the distance between the distributions of the actions taken by the shielded and the nominal policies. Thereby, the shielding cost can also vary based on the history of the run, allowing for a more fine-grained analysis as in~\cite{chatterjee2023shielding}.
To formalize this, we define a history-based cost function $\cost:\fruns \rightarrow [0,W]$ that assigns a cost to the history $\hist$ and the cost of shielding the policy $\policy$ at $\hist$ is defined as $\cost(\hist,\policy,\shield{\policy}) = \cost(\hist)\cdot \distance(\policy(\hist),\shield{\policy}(\hist))$.
This cost captures the interference as the difference between the action distributions of the original and shielded policies based on the cost of shielding at $\hist$.
This can be generalized to the cost of a run $\rho$ as 
$\cost(\rho,\policy,\shield{\policy}) = \limsup_{l\to\infty} \frac{1}{l}\sum_{i=0}^{l-1} \cost(\rho[0;i],\policy,\shield{\policy})$,
capturing the average cost of the difference between both policies over a run.

In the following, 
we show that the expected average cost of the shielded policy is bounded whenever the template considered only contains liveness templates.
These restrictions are needed because the shielded policy may have to stop (or eventually stop) taking certain actions—such as unsafe or co-live ones—to satisfy the $\omega$-regular constraint. 

\begin{restatable}{theorem}{restateMinimalInterferenceWeighted}\label{thm:minimalinterferenceWeighted}
Given the premises of \cref{corollary:shieldedPolicy} with $\shield{\policy} = \policy|^{\template_\psure,\thres}_{\gamma}$, $\template_\psure = (\emptyset,\emptyset,\livegroup)$, and cost function $\cost:\fruns^M \rightarrow [0,W]$, for any $\varepsilon > 0$, suitable $\gamma,\thres > 0$ exist such that $\exp_{\rho\sim\shield{\policy}} \cost(\rho,\policy,\shield{\policy}) < \varepsilon$.
\end{restatable}

We note that the corner case discussed in \cref{rem:cornercase} does not compromise these results. In \cref{thm:minimalinterference}, any history $\hist$ satisfying $\hist \models \pref(\Phi)$ inherently avoids unsafe actions. In \cref{thm:minimalinterferenceWeighted}, the assumption that the strategy template excludes unsafe and co-live actions ensures that the bound on the expected average cost remains valid.

\subsection{Dynamic Adaptations of \Stars{}}\label{sec:incremental}%
So far, we have considered a shielding scenario for a \emph{static} parity objective $\Phi$. 
However, a major strength of strategy templates is their \emph{efficient compositionality} and \emph{fault-tolerance}, which allow for further dynamic adaptations of \Stars{}. 

\emph{Compositionality} facilitates the incremental integration of multiple $\omega$-regular specifications into \Stars{}. By using the existing algorithm \textsc{ComposeTemplate} from~\cite[Alg.4]{AnandNS23} we can compute \Stars for \emph{generalized parity constraints} of the form $\Phi = \land_{i=1}^k \Phi_i$, where each $\Phi_i$ represents a parity constraint over $G_\psure^M$. Crucially, these objectives $\Phi_i$ may not be available all at once but might arrive incrementally over time, leading to the need to update the applied shield at runtime. As \textsc{ComposeTemplate}  simply combines strategy templates for all objectives into a single (non-conflicting) template, \cref{thm:shieldtemplate} and \cref{thm:minimalinterference} also apply in this case, as long as the run is in the combined winning region of all objectives during the update. 

\emph{Fault-tolerance} ensures that \Stars{} can handle the occasional or persistent unavailability of actions correctly. Concretely, persistent faults are addressed by marking actions as unsafe and resolving conflicts as needed~ (see~\cite[Alg.5]{AnandNS23}), while occasional faults are handled by temporarily excluding the unavailable actions from the template~ (see~\cite[Sec.5.2]{AnandNS23}).

\begin{remark}\label{rem:incremental}
A common assumption in robotic applications is that (incrementally arriving) \emph{liveness} specifications are satisfiable from every safe node in the workspace~\cite{GundanaKressGazitRobots21,KantarosPappasEtAllRobots21, KallurayaPappasKantarosRobots2023,PacheckKressGazitRobots23,ZiyangEtAlRobots24} -- most robotic systems can simply invert their path by suitable motions to return to all relevant positions in the workspace.
Using the common decomposition of $\omega$-regular objectives $\varphi$ into a safety part 
$\varphi_s$ and a liveness part $\varphi_\ell$, one can restrict the newly arriving objectives to  liveness obligations only. In this case, incremental synthesis never leads to an decreased winning  region in such robot applications. This is in fact the case in the incremental  instances considered in the \benchmark benchmark used for evaluation.
\end{remark}

\subsection{Optimality of \Stars{} in Rewardful MDPs}\label{sec:optimality}
While we already formalized \emph{correctness and minimal interference} of \Stars{}, now we 
strengthen this result further, i.e., we show that whenever nominal policies have been computed 
to optimize a given reward function (under certain assumptions), 
\Stars{} produce a shielded policy which achieves a reward which is $\epsilon$-close to the nominal one while additionally guiding the agent to (almost) surely satisfy an $\omega$-regular correctness specification. 

An optimal policy over an MDP is typically computed (e.g.\ via reinforcement learning (RL)) by associating transitions with \emph{reward functions} which are typically Markovian and assign utility to state-action pairs. Formally, a \emph{rewardful MDP} is denoted as a tuple $(M,r)$ where $M$ is an MDP equipped with a reward function $r : Q \times A \mapsto\R$. 
Such an MDP under a policy $\policy$ determines a sequence of random rewards $r(X_i,Y_i)$ for $i \geq 0$, where $X_i$ and $Y_i$ are the random variables denoting the $i^{th}$ state and $i^{th}$ action, respectively.
Given a rewardful MDP $(M,r)$ with initial state $\qinit$ and a policy $\policy$, we define the \emph{discounted reward} and the \emph{average reward} 
    \begin{align*}
    &\textstyle
     \discreward{\lambda}{\qinit}{\policy}:=\lim_{N \to\infty} \exp^{\qinit}_{\policy} \left( \sum_{0 \leq i \leq N} \lambda^i r(X_i,Y_i) \right)
     \\
        &\textstyle
     \avgreward{\qinit}{\policy}:=\limsup_{N \to\infty} \frac{1}{N} \exp^{\qinit}_{\policy} \left( \sum_{0 \leq i \leq N} r(X_i,Y_i) \right).\label{equ:average}
    \end{align*}

For discounted rewards, the impact of obtained rewards decreases with time. Therefore, the optimal reward achievable by a policy over a given MDP 
significantly depends on the bounded (initial) executions possible over this MDP. Hence, the $\epsilon$-closeness follows directly from \cref{thm:minimalinterference}.

For average rewards, the optimal average reward is equivalently impacted by rewards collected over the entire (infinite) length of runs compliant with the policy. This implies, that a policy can only satisfy an $\omega$-regular property (almost) surely \emph{and} optimize the average reward, if it can \enquote{switch} between their satisfaction by alternating infinitely often between finite intervals which satisfy either one. 
This, however, is only possible if the underlying MDP is `nice' enough to allow for this alternation.
In particular, it is known that optimal average rewards is usually obtained by taking the maximum over the achievable rewards in each \emph{good-end components} of the MDP (see~\cite{AlmagorKV16} for details). 
Therefore, the $\epsilon$-closeness of the average reward can be guaranteed when the MDP is a good-end component for the specification.

\begin{theorem}\label{thm:optimality}
  Given the premises of \cref{corollary:shieldedPolicy} with $\shield{\policy} = \policy|^{\template_\psure,\thres}_{\gamma}$ and $\win^{\psure}_\Phi = Q$, for every $\varepsilon > 0$, with suitable parameters $\thres,\gamma > 0$ it holds that
    $\discreward{\lambda}{\qinit}{\shield{\policy}} > \discreward{\lambda}{\qinit}{\policy} - \varepsilon$.
  Furthermore, if the MDP is a good-end component, then with suitable parameters, it holds that
    $\avgreward{\qinit}{\shield{\policy}} > \avgreward{\qinit}{\policy} - \varepsilon$.
\end{theorem}

\begin{remark}\label{rem:WequalQ}
 We remark that the assumption $\win^{\psure}_\Phi = Q$ in \cref{thm:optimality} is not restrictive for two main reasons. First, if $\win^{\psure}_\Phi \subsetneq Q$, we can use $\win^{\psure}_\Phi$ as an additional constraint in existing safe reinforcement learning frameworks. For instance, one can use \emph{preemptive safety shielding} proposed in~\cite{alshiekh2018safe}. This allows to learn an optimal policy within $\win^{\psure}_\Phi$ which directly allows to transfer the results from \cref{thm:optimality}. 
 Second, we recall the discussion of \cref{rem:incremental} to note that including the safety-part of the objective into the learning process does not harm the incremental adaptation of \Stars{} when new (liveness) specifications arrive.
\end{remark}

\section{Experiments}
\label{sec:experiments}
We empirically evaluate STARs on three benchmarks introduced in \Cref{sec:introLuna} to demonstrate their effectiveness in: (i) enforcing $\omega$-regular (including liveness) properties at runtime, (ii) enabling dynamic specification updates without retraining, and (iii) scaling to complex environments.
To showcase this, we have implemented our shielding algorithm \applyStars{} (as depicted in~\cref{fig:overview}) in a Python-based prototype tool \tool{} (Monitoring and Adaptive Runtime Guide). %
The experiments were 
run on a 32-core Debian machine with an Intel Xeon E5-V2 CPU (3.3 GHz) and up to 256 GB of RAM.
Recordings of some simulations are available at 
{\textcolor{blue}{\url{http://anonymous.4open.science/w/MARGA}}}.

\begin{figure*}[t]
\centering
	\begin{subfigure}[t]{0.32\textwidth}
		\includegraphics[width=1.1\textwidth]{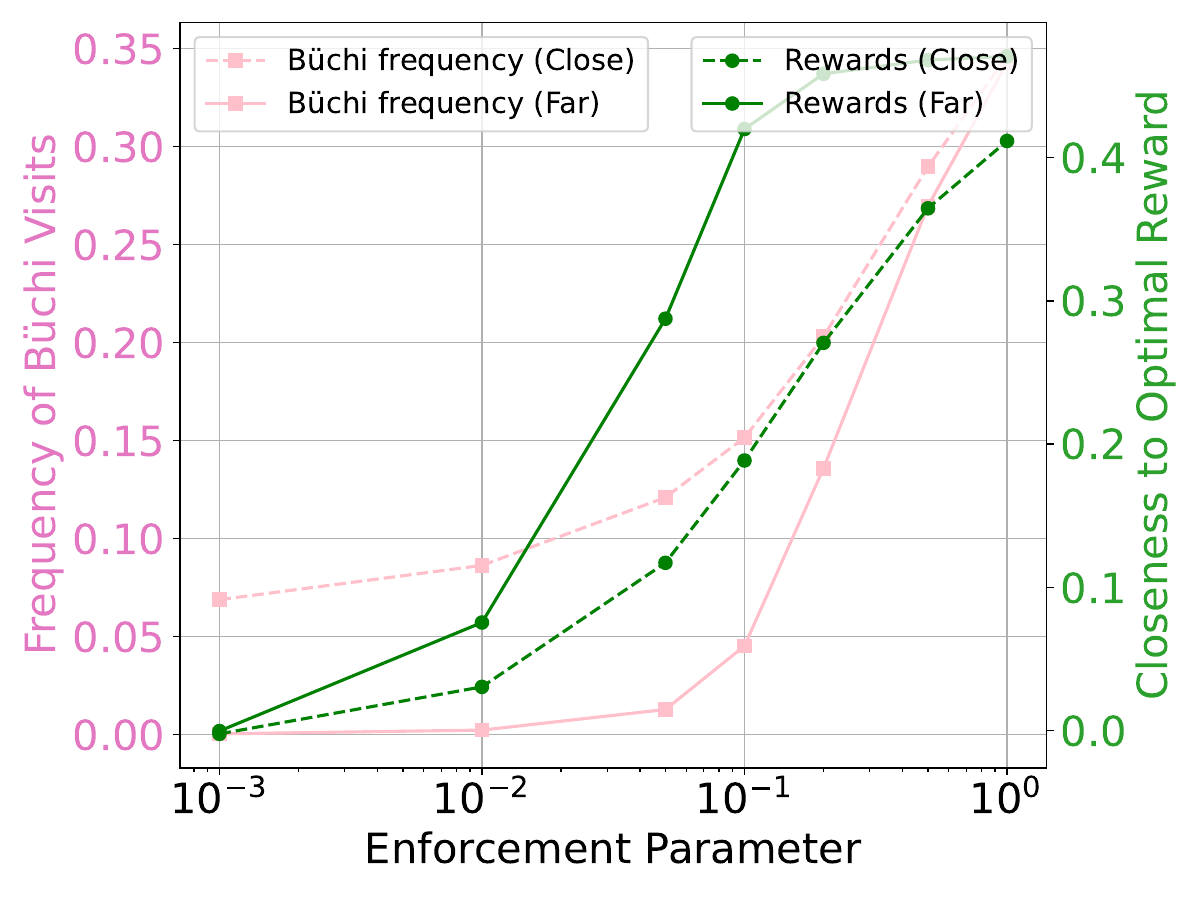}
		\caption{}
		\label{fig:enforcementParameterTrend}
	\end{subfigure}
	\hfill
	\begin{subfigure}[t]{0.32\textwidth}
		\includegraphics[width=1.2\textwidth]{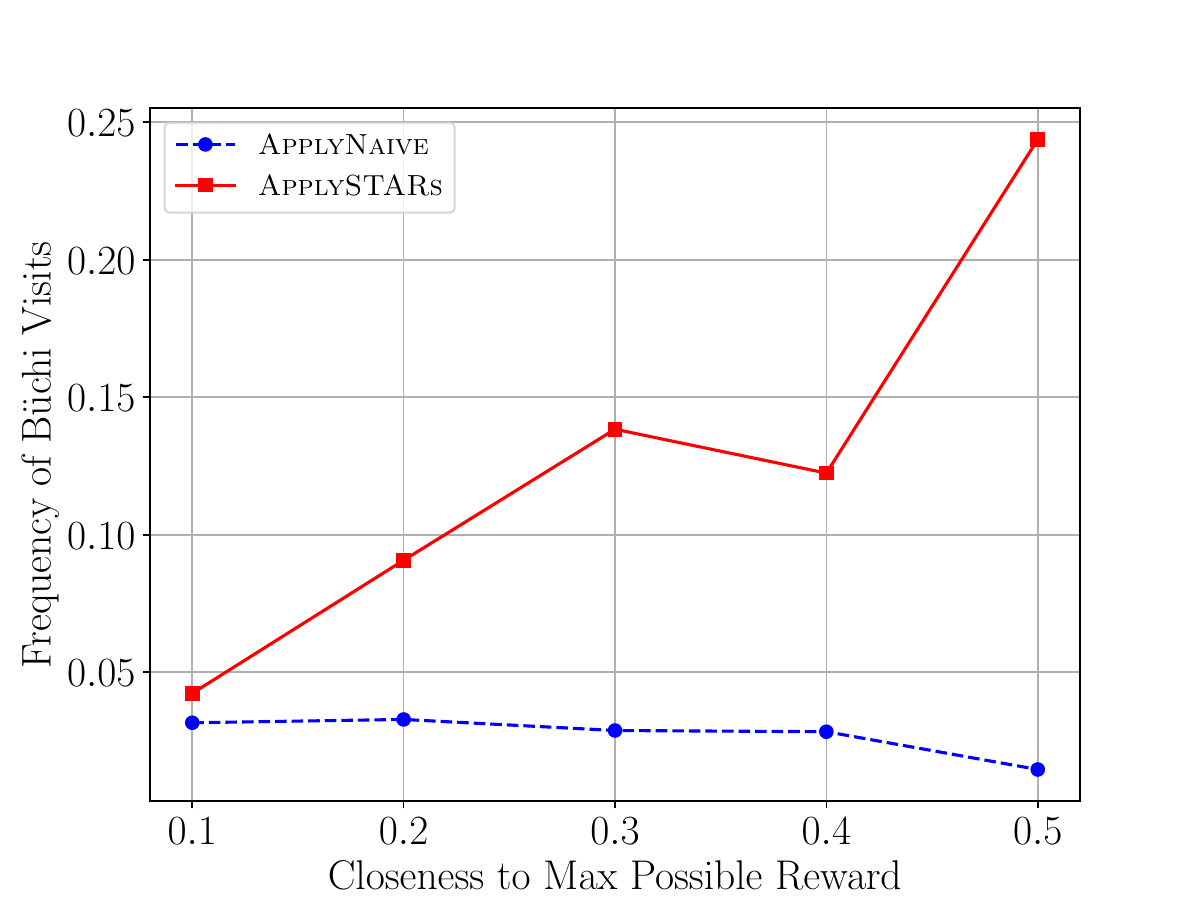}
		\caption{}
		\label{fig:thresholdVsBuechi}
	\end{subfigure}
	\hfill
	\begin{subfigure}[t]{0.25\textwidth}
		\includegraphics[width=\textwidth]{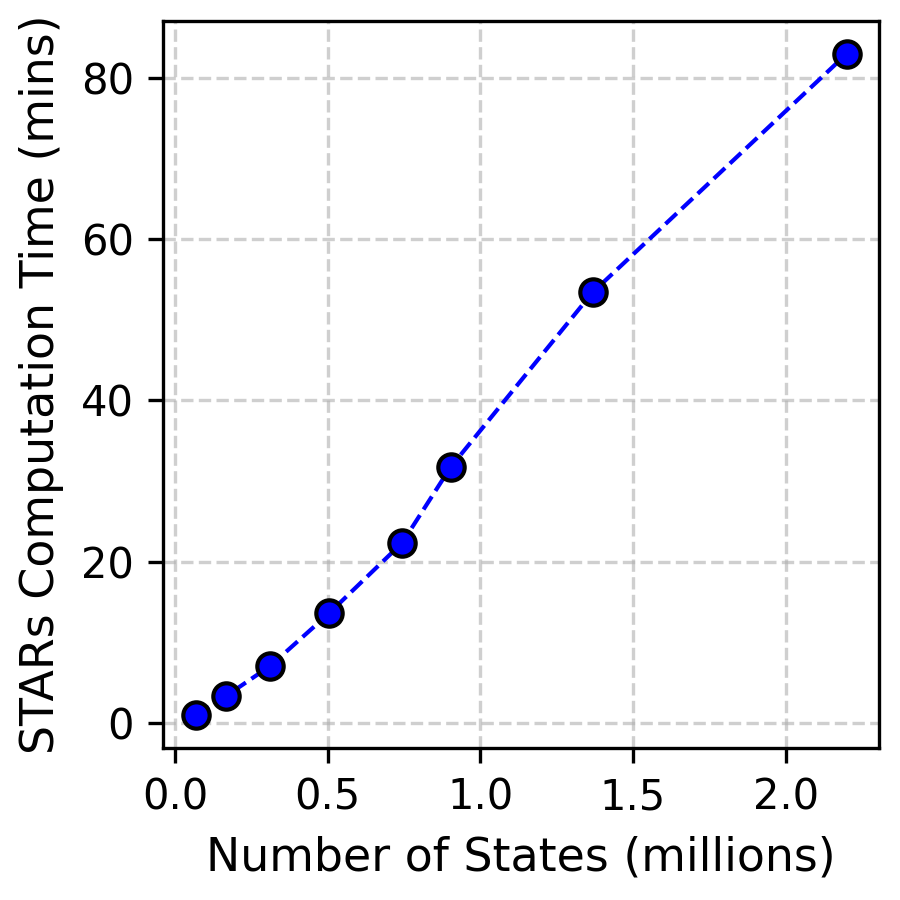}
		\caption{}
		\label{fig:overcooked}
	\end{subfigure}
	\vspace{-0.3cm}
	\caption{Evaluation summary. (a) Effect of $\gamma$ on the \buchi frequency and average reward for \applyStars in \benchmark{} (b) \buchi frequency vs. average reward for \applyStars and \applyNaive (c) Scalability of \applyStars on Overcooked-AI.}
	\label{plot:experiments}
\end{figure*}
\begin{figure}
	\begin{subfigure}[t]{0.4\textwidth}
		\includegraphics[width=\textwidth]{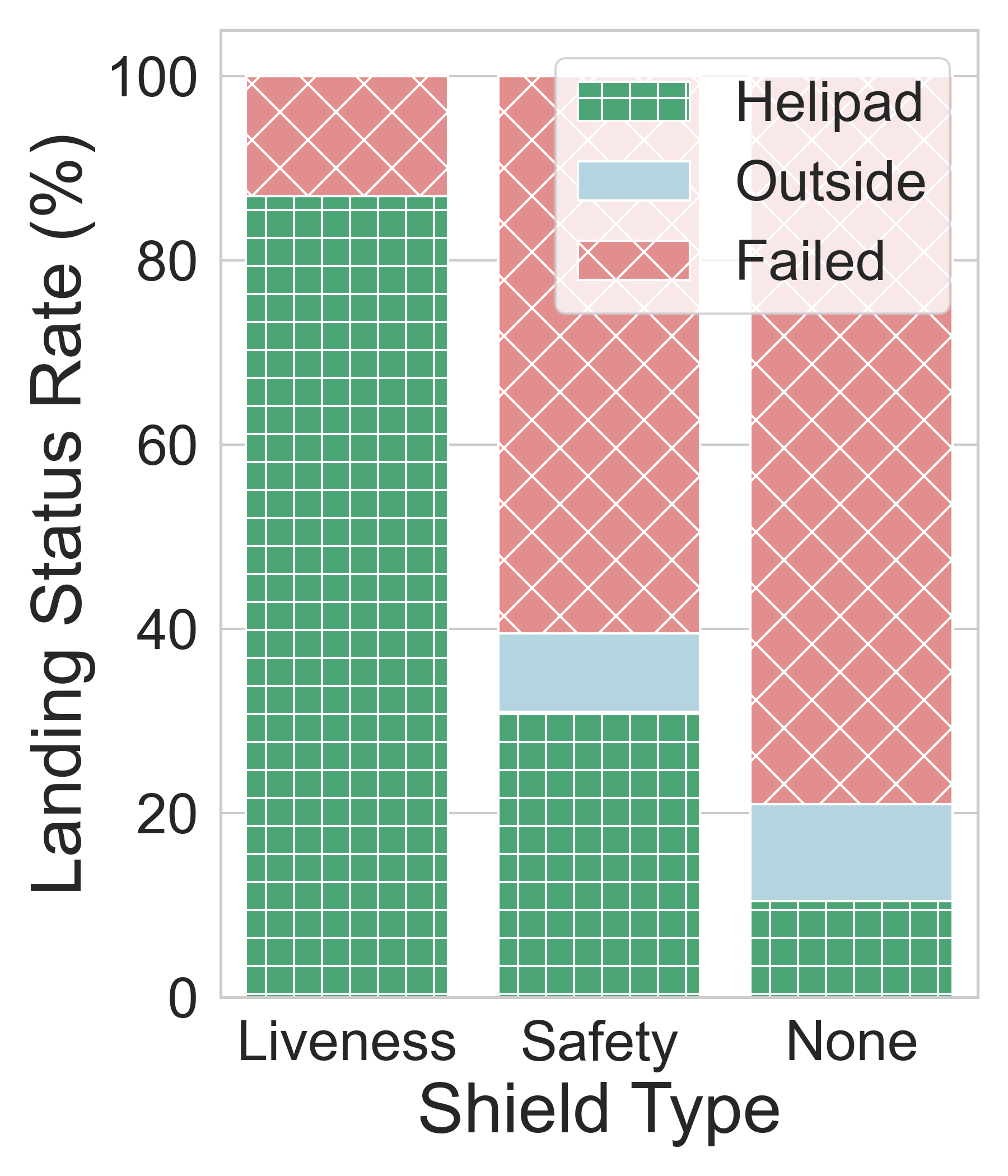}
		\caption{}
		\label{fig:subfig3}
	\end{subfigure}
	\hfill
	\begin{subfigure}[t]{0.4\textwidth}
		\includegraphics[width=\textwidth]{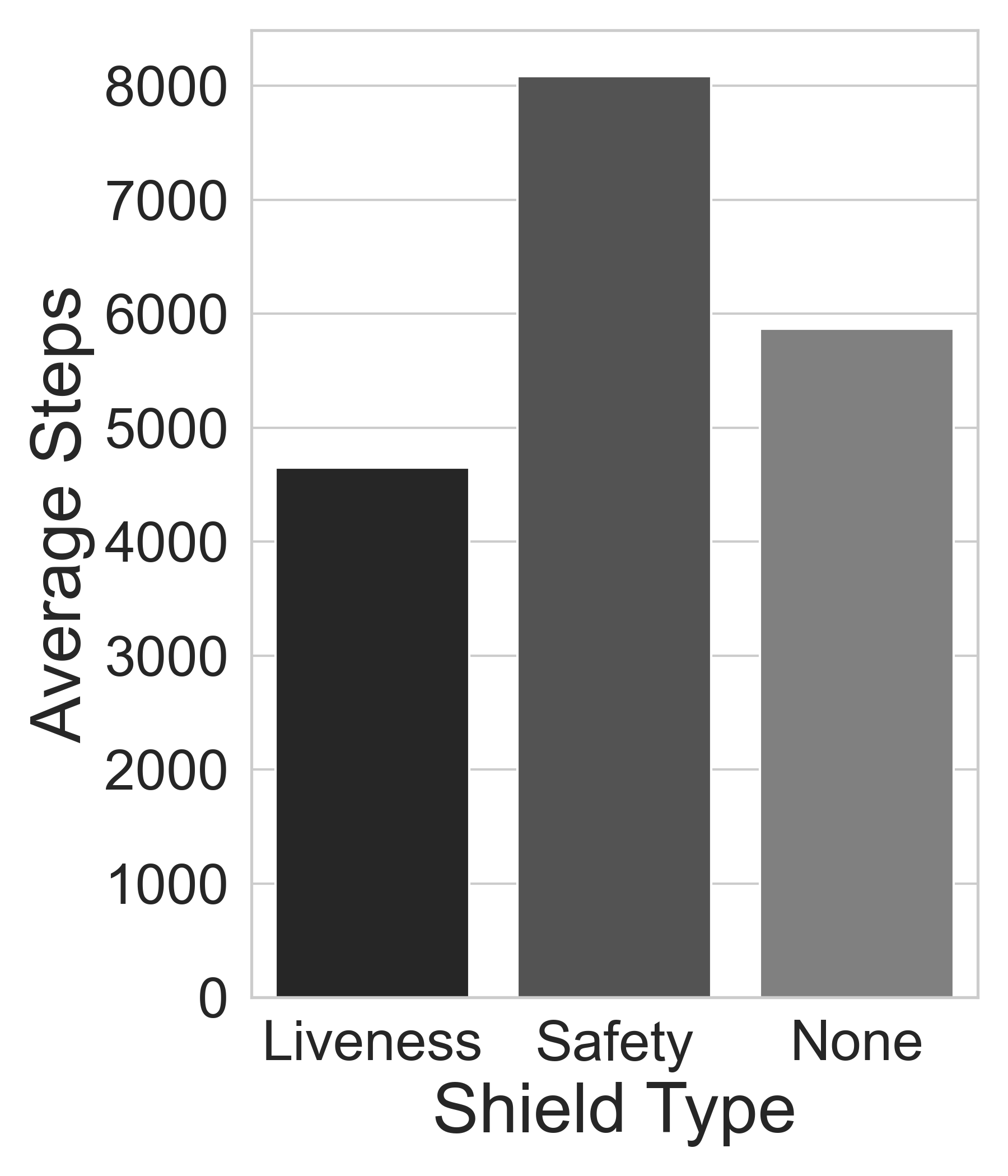}
		\caption{}
		\label{fig:subfig4}
	\end{subfigure}
	\vspace{-0.3cm}
	\caption{(a) Rate of different landing outcomes in LunarLander (b) Average steps to successfully land in LunarLander.}
	\label{fig:LunarLander}
\end{figure}

\subsection{\benchmark{} Benchmarks}
We evaluate our tool on a benchmark suite \benchmark{} for guiding a robot in a grid world (see \cref{screenshot:botunshielded,screenshot:botShielded}). As noted in \cref{sec:introLuna}, the robot follows a policy $\policy$ trained to maximize average reward on randomly generated grids and then shielded with dynamic safety and liveness specifications (e.g., goal visiting, obstacle avoidance) using our tool. These specifications are programmatically modified to simulate deployment-time constraint changes.

Unfortunately, the closest liveness shielding tool from~\cite{chatterjee2023shielding} is not publicly available. Moreover, our evaluation focuses on dynamic interference via online tuning of the enforcement parameter $\gamma$, a feature not supported by existing methods. We also discuss a comparison against a naïve shielding baseline.

\smallskip\noindent\textbf{Experimental Setup.}
We generate random square grids (with sizes from 5 to 13) by sampling walls and reward zones. The reward region is placed at an $\ell_1$-distance between min_dist·size and max_dist·size from every cell in the Büchi region (\buchiRegion). We construct 383 instances: 189 \far (min_dist = 0.7, max_dist = 0.9) and 194 \close (min_dist = 0.1, max_dist = 0.2) instances.
We start with a policy that maximizes average reward without the \buchi objective, then \applyStars is applied. For each instance, we measure the number of \buchi region visits and the average reward over 100,000 steps from a random initial state.

\smallskip\noindent\textbf{Tuning $\gamma$.}
We analyze the trade-off between the frequency of visits to the \buchi region (\buchi frequency) and the average reward achieved by \applyStars. To achieve higher average rewards, the robot must spend more time in the high payoff region (\payoffRegion). We evaluate this trade-off by systematically varying the enforcement parameter $\gamma$ in \applyStars and, for each value, measure the average \buchi frequency and average reward across all instances that achieve an average reward within $\varepsilon$-close (for $\varepsilon\in\{0.1, \ldots, 0.5\}$) of the maximum possible for that instance. This evaluation is performed separately for the \far and \close categories. As the robot needs to remain longer in \payoffRegion{} to increase the average reward, the distance between \buchiRegion and \payoffRegion{} becomes a key factor: as the distance increases,
$\gamma$ needs to be smaller to attain a given closeness $\varepsilon$ to the average reward.

This theoretical dependence is supported by \cref{fig:enforcementParameterTrend} which shows the \buchi frequency (pink) and the proximity to the maximum 
average reward (green) attained by \applyStars for a given enforcement parameter $\gamma$, over  instances 
from \far (dashed) and \close (solid), respectively. 
We observe that as the enforcement parameter increases, the 
\buchi frequency increases while the average reward gets further away from the optimal for both classes of instances. As expected, these trends have a higher slope on \far instances.

\begin{remark}
We report only on experiments with optimal average reward policies, as discounted rewards pose less challenge for shielding. Discounted reward policies depend on initial trace segments, while $\omega$-regular objectives can be satisfied regardless of any finite prefix (if started in the winning region). This enables a trivial shielding strategy in the \benchmark benchmark: use a low $\gamma$ initially, then increase it later. While this yields high performance for both objectives, it is only possible because \Stars support dynamic $\gamma$ updates at runtime—a key advantage over existing methods.
\end{remark}

\smallskip\noindent\textbf{Quality of Shielded Policies.}\label{sec:naive}
Given the fact that \cref{thm:optimality} restricts attention to \emph{good end components} (i.e. the maximum color in the end component is even), we note that any stochastic policy $\policy$ satisfies $\Phi$ almost surely within each good end component $\widetilde{Q}$ already without any shielding. This is due to the fact that under stochastic policies all edges have positive probability of being sampled and, therefore, infinite runs reach all states in the SCC almost surely. As the maximum color in $\widetilde{Q}$ is even, all runs satisfy $\Phi$ almost- surely. 
In practice, however, the frequency with which even color vertices are seen is extremely low. 
While one might suggest that perturbing the nominal policy $\policy$ might increase the frequency of visiting even color states, this is actually not the case, as this perturbed policy would explore the \emph{entire state space} more aggressively.
We call the algorithm that implements this perturbation $\applyNaive$.

In contrast, \Stars modify probabilities in a targeted fashion. This (i) avoids visiting odd-color vertices which are not optimal, and (ii) allows to tune the desired frequency in which even color vertices are visited via the enforcement parameter. 
This can be formalized using the notion of \emph{frequency} of a run $\run$ visiting a set $T$ of states which can be defined as $\freq(\run,T) = \limsup_{l \to \infty} \frac{1}{l} \abs{\{i \in [0;l] \mid \run[i] \in T\}}$. With this definition, 
the following theorem\footnote{
Note that this shows the existence of such parameters only for the case of surely satisfying $\Phi$. For the case of almost-sure satisfaction, the frequency would also depend on the transition probabilities of the MDP and hence, we cannot guarantee the existence of such parameters for every $\delta$, while the same intuition still holds.}
ensures that the frequency of a run $\run$ visiting even color states can be increased by tuning the enforcement parameter $\gamma$.
\begin{restatable}{theorem}{restateFrequency}\label{thm:frequency}
    Given the premises of \cref{thm:optimality} with $\shield{\policy} = \policy|^{\template_\sure,\thres}_{\gamma}$ and $T$ being the set of even color states in the objective $\Phi$, for every frequency $0<\delta\leq \frac{1}{\abs{Q}}$, there exists parameters $\thres,\gamma > 0$ such that for every run $\run\sim\shield{\policy}$, it holds that $\freq(\run,T) \geq \delta$.
\end{restatable}

This theoretical result is supported by \cref{fig:thresholdVsBuechi} where we compare our shielding approach \applyStars with the naive shielding approach \applyNaive (discussed above) on our \benchmark{} benchmark suite. %
\cref{fig:thresholdVsBuechi} plots the 
average \buchi frequency (y-axis) for all instances that obtained an average reward that is $\varepsilon$-close to the maximal possible reward in that instance, where $\varepsilon\in\{0.1, 0.2, 0.3, 0.4, 0.5, 0.6\}$ (x-axis). The red line 
represents the averages for the robot shielded by \applyStars, and the dashed blue line represents the same for the one shielded by \applyNaive. We observe that 
\applyStars can maintain a similar average reward as \applyNaive, while ensuring a much higher \buchi 
frequency. In addition, \applyNaive does not allow to increase the \buchi frequency beyond a very low level. On the other 
hand, by sacrificing on the average reward, \applyStars allows attaining very high \buchi 
frequency.

\subsection{Overcooked-AI Benchmarks}
\label{sec:overcooked}

We assess the scalability of our shield computation via the Overcooked-AI environment~\cite{carroll2019utility}, a widely used benchmark for cooperative multi-agent reinforcement learning.  Here, autonomous agents are trained to repeatedly perform cooking tasks. We use LTL specifications to encode additional recipe requirements of produced dishes. \Stars{} are used to enforce the (additional) production of soups satisfying these requirements infinitely often. 

From a shielding perspective, Overcooked-AI is very similar to \benchmark as it is based on a known finite MDP and the LTL recipe objective gets translated into a Büchi objective on the two-player game graph extracted from the known MDP. Therefore, the insights about tuning the enforcement parameter and improving shielding quality discussed in \cref{sec:naive} identically apply to Overcooked-AI. Here 'far' and 'close' instances (cf. Fig.~\ref{fig:enforcementParameterTrend}) however relate recipes rather then grid states. If additional recipe constraints align with the recipe the agent policy was learned on no shielding is needed and both types of dishes are produced with a similar frequency. If recipes are complementary, $\gamma$ can be used to balance the frequency of produced dishes. 

The main difference between \benchmark and Overcooked-AI is the size of the naturally arising game arenas which is why we choose Overcooked-AI benchmarks to demonstrate the scalability of \applyStars. Towards this goal, we utilize with the \texttt{cramped\_room} layout with increasing sizes. Within this layout, two agents can move in four directions and perform actions such as picking up ingredients, serving soups, and placing items on counters. %
We construct the game graph naively by enumerating all possible states and actions, resulting in graphs ranging from 68000 to 2.2 million states across different layout sizes. We compute \Stars for each instance, with an overview of state counts and computation times presented in \cref{fig:overcooked}. These results demonstrate that the synthesis of \Stars scales to million-state environments with practical one-time computation time cost (about 1 hour) prior to deployment. We suspect that improved game graph extractions that cluster states with similar shielding requirements can further improve the scalability of \applyStars and is an interesting direction for future work.

\subsection{LunarLander Benchmarks}
We finally evaluate \applyStars on the LunarLander benchmark~\cite{brockman2016openai}. 
The standard environment is modified to have more height and to have a randomly positioned 
helipad. We also modify the reward to not depend on landing on the pad\footnote{The details of the exact modifications to the standard environment are described in the appendix.}.

\smallskip\noindent\textbf{Experimental Setup.}
We train a baseline policy using standard proximal policy optimization (PPO) for 50000 time steps, which ensures that the lander touches down safely---though not necessarily on the helipad. Then, to apply \Stars, we introduce a $60 \times 60$ grid which discretizes only the $x/y$ coordinates of the hidden 8-dimensional MDP of the lunar lander. We further reinterpret the lander actions to move along those grid cells, which vastly simplifies the actual dynamics of the lunar lander. We extract a game from this grid equipped with a safety (preventing the lander from leaving the environment) and a liveness (steering toward the helipad) objective and use it to synthesize STARs. We 
then \applyStars on the trained policy. For comparison, we also evaluate a safety shield which solely prevents the lander from leaving the environment or landing outside the helipad region.

\smallskip\noindent\textbf{Results.}
We randomly generated 200 seeds to initiate the environment, and compare the result of policies being (i) unshielded, (ii) shielded with safety shield, and (iii) shielded with \Stars{}. The results are summarized in \cref{fig:LunarLander}. The results show that the unshielded policy lands on the helipad in $10.5\%$ cases, while the safety shielded policy lands on the helipad in $31\%$ cases, and our shielded policy lands on the helipad in $87\%$ cases. Also, the average steps to land on the helipad are 5868 for unshielded, 8082 for safety shielded, and 4650 for liveness shielded policies showing that liveness steers the lander toward the helipad more quickly. This showcases the effectiveness of \Stars{} in ensuring liveness while preserving the learnt policy to not crash the lander.

In particular, this behavior is observed even though the actual dynamics of the lander are very complex, leading to quite different abstract behavior as captured by our small game abstraction used to compute \Stars. This showcases the applicability of \Stars to complex high-dimensional environments. Combined with the scalability results on Overcooked-AI we believe that a future scaling of \Stars towards industrial scale benchmarks is possible.

\label{beforebibliography}
\newoutputstream{pages}
\openoutputfile{main.pg}{pages}
\addtostream{pages}{\getpagerefnumber{beforebibliography}}
\closeoutputstream{pages}

 \bibliography{main}
 
 \newpage
 \appendix

\section{Additional Related Work}
\paragraph*{Permissive Strategies.}
While the synthesis of permissive strategies for $\omega$-regular objectives has received substantial attention in recent years~\cite{bernet2002:permissiveStrategies,bouyer2011:measuringpermissiveness,FremontSeshiImprovisation,Klein2015:mostGeneralController,AnandNS23}, and strategy templates~\cite{AnandNS23,AnandNS24} have been applied to various problems in reactive synthesis~\cite{SchmuckHDN24,AnandMNS23,NayakS24,AnandSN24,NayakEGJS23,phalakarn2024winningstrategytemplatesstochastic}, these techniques have, to the best of our knowledge, not been used in the context of shielding.

\section{Additional Preliminaries}\label{sec:prelims}
\paragraph*{Finite Memory Policies.}
Let $\mem$ be a set called \emph{memory}. A policy $\policy$ with memory $\mem$ is represented as a tuple $(\mem, m_0, \alpha, \beta)$ where $m_0 \in \mem$ is the initial memory value, $\alpha : \mem \times Q \mapsto\mem$ is the memory update function, and $\beta: \mem \times Q \mapsto\dist{A}$ is the function prescribing the distribution over the next set of available actions. A policy $\policy$ is said to be a \textit{finite memory} policy if $\mem$ is a finite set. It is called \emph{stationary} if $\mem = \emptyset$, i.e., the choice of action only depends on the state. 
Given a finite run (or history) $\hist$, a state $q$ and an action $a \in A(q)$, $\policy(\hist q, a) = \pr(a|\hist q)$ denotes the probability that $\policy$ assigns for choosing the action $a$ from state $q$ with history $\hist$. If $\policy$ is stationary, we will write $\policy(q, a)$ instead of $\policy(\hist q, a)$.
Given a random variable $f : \runs^M \mapsto \R$, we denote by $\exp^q_{M^\policy}(f)$ the expectation of $f$ over the runs of $M$ originating at state $q$ that follow $\policy$. We instead write $\exp^q_{\policy}(f)$ when $M$ is clear from the context.

\paragraph*{Stochastic Graph Games.}
A \emph{stochastic game graph} is a tuple $G = (Q = \Qsys \cup \Qenv \cup \Qavg, E)$ where $(Q,E)$ is a finite directed graph. For every state $q \in Q$, we denote the set of all available edges from $q$ as $E(q)$ and assume $|E(q)|>0$ for all $q\in Q$. Further, for $\player\in\{\sys,\env,\avg\}$, we define $E_{\player}= \{(q,q') \in E \mid q \in Q_{\player}\}$. 

A stochastic game involves three players: `system' ($\bigcirc$), `environment' ($\square$), and `random' ($\triangle$).  They take turns moving a token along states, forming a path. When the token is at a state in $\Qsys$ (resp. $\Qenv$), the system (resp.\ environment) player chooses one of its successors to move the token. At a state in $\Qavg$, the random player moves the token to one of its successors following a known or unknown probability distribution, selecting uniformly at random. Stochastic game graphs are often called $2\thalf$-player game graphs. If $\Qavg = \emptyset$, $\Qenv = \emptyset$, or $\Qenv=\Qavg=\emptyset$, they reduce to $2$-player, $1\thalf$-player, and $1$-player game graphs, respectively. Game graphs without a random player are called \emph{deterministic}.

Given a stochastic game graph $G$, a \emph{run} (or play) $\rho$ over $G$ is an infinite sequence of states $q_0 q_1 \ldots\in Q^\omega$. 
We write $\inf_Q(\rho)$ (resp.\ $\inf_E(\rho)$) to denote the set of all states (resp. edges) which occur infinitely often along $\rho$.
We collect all runs over $G$ in the set $\runs^G$. 
A \emph{strategy} for player $\player\in\{\sys,\env,\avg\}$ over $G$ is a function $\policy_\player : Q^* \times Q_\player \to \mathcal{D}(Q)$ that describes a probability distribution over next available moves to the successor states based on the history of the current run. 
Given a system player strategy $\policy$, a run $\rho$ is said to comply with $\policy$, i.e., be a $\policy$-run, if $q_{i+1} \in \supp(\policy(q_0\ldots q_i))$ holds for all $q_i \in \Qsys$ along $\rho$. Given a measurable set of infinite runs $P \subseteq \runs^G$, $\Pr_{\strat}[P]$ is the probability that a $\strat$-run belongs to $P$. We use $\runs^{G^\policy}$ to denote the set of all $\policy$-runs over $G$.

\paragraph*{MDPs and Stochastic Games.}
A standard reduction from an MDP with an $\omega$-regular objective to a stochastic parity game involves treating the randomness by a second player, either as an adversarial environment player (for sure satisfaction) or as a random player (for almost-sure satisfaction). Note that, in both cases, the game can be constructed without access to the probabilities $\Delta$ of the MDP.
This is formalized as follows.
\setcounter{definition}{3}
\begin{definition}
 Given an MDP $M=\langle Q, A, \Delta, \qinit\rangle$ we define the $2$-player (resp. $1\thalf$-player) game graph induced by $M$ as the tuple $G_\sure^M = (\Qsys \cup \Qenv, E_\sys\cup E_\env)$ (resp. $G^M = (\Qsys \cup \Qavg, E_\sys\cup\ E_\avg)$) s.t.\ for $\player\in\{\env,\avg\}$:
 \begin{compactitem}
  \item $\Qsys:=Q$,\quad $Q_\player:=\{q_\player^a\mid a\in A(q) \text{ and } q\in Q\}$,
  \item $E_\sys:=\{(q,q^a_\player)|a\in A(q)\}$, and $E_\player:=\{(q^a_\player,q')|\pr(q'|q,a)>0\}$.
 \end{compactitem}
\end{definition}

We remark that assuming $\Phi$ to be directly defined over $G^M$ is not restrictive. Any $\omega$-regular property $\varphi$ with propositions interpretable as subsets over $Q$ can be converted into a parity game $(G_\varphi,\Phi_\varphi)$ which can be combined with $G^M$ through a simple product.

\paragraph*{Conflict-Freeness of Strategy Templates.}
Note that the existing algorithms for computing strategy templates~\cite{AnandNS23,phalakarn2024winningstrategytemplatesstochastic} always produce \emph{conflict-free} strategy templates.
A strategy template $\strat$ is called \emph{conflict-free} if,
for every state $q \in Q$, there exists an action $a \in A(q)$ such that $(q,a) \not\in \safegroup$ and $(q,a) \not\in \colivegroup$, and for every source $q$ of live-group $H \in \livegroup$, there exists an action $a$ in $H$ such that $(q,a) \not\in \safegroup$ and $(q,a) \not\in \colivegroup$.
This is crucial, as it allows us to obtain a strategy easily from a strategy template (see~\cite{AnandNS23} for details).
For simplicity, throughout the paper, we assume that strategy templates are conflict-free.

\section{Proof of Results in Dynamic $\omega$-Regular Shielding}
\subsection{Correctness of \Stars{}}
We first restate the correctness of \Stars{} as in Theorem~1 of the main text and then prove it.
\begin{restatable}{theorem}{restateShieldTemplate}
    \label{thm:shieldtemplate}\label{corollary:shieldedPolicy}
    Given any MDP $M$, a strategy template $\template$ interpreted over $M$, a threshold $\thres$, and enforcement parameters $\gamma ,\theta > 0$, the shielded policy $\policy|^{\template_\psure,\thres}_{\gamma}$ follows $\template_\psure$.
    Moreover, if $\template_\psure:=\parityTemp_\psure(M,\Phi)$, then, every $\policy|^{\template_\psure,\thres}_{\gamma}$-run from the winning region of $\Phi$ satisfies $\Phi$ surely/almost surely (depending on $\psure$).
\end{restatable}
\begin{proof}
    Let $\shield{\policy}:=\policy|^{\template_\psure,\thres}_{\gamma}$ (for notational convinience) and $\run$ a $\shield{\policy}$-run of $M$. We show that $\run$ satisfies the template $\template_\psure = (\safegroup,\colivegroup,\livegroup)$.
    As probability of an unsafe edge $\edge{q}{a}\in\safegroup$ is set to zero by $\shield{\policy}$, the safety template $\safegroup$ is satisfied.
    
    Let's assume that $\run$ does not satisfy the co-live template $\colivegroup$. Then, there exists a co-live edge $e= \edge{q}{a}$ that appears in $\run$ infinitely many times. 
    Let $\hist q$ be a finite prefix of $\run$ such that $a$ has been sampled from $q$ more than $1/\gamma$ times. Then, $\counter{\edge{q}{a}}(\hist q) > 1/\gamma$, which means $\policy(\hist q, a) - \gamma \cdot \counter{\edge{q}{a}}(\hist q) < 0$. Consequently, the probability of choosing $e$ after history $\hist q$ under $\shield{\policy}$ becomes zero. Hence, $\run$ can not visit $q$ more than $1/\gamma$ times, which contradicts the assumption that $\run$ visits $(q,a)$ infinitely many times. Thus, $\run$ satisfies the co-live template $\colivegroup$.

    Next, suppose $\run$ does not satisfy the live-group template $\livegroup$. Then, there exists a live group $\livegroupSingle \in \livegroup$ such that $\run$ visits the source states of $\livegroupSingle$ infinitely many times but does not sample any action from $\livegroupSingle$ infinitely many times. 
    Let $q \in\src(\livegroupSingle)$ be a state that is visited infinitely many times by $\run$, and let $q$ be a source state for live-groups $\livegroupSingle_1,\ldots,\livegroupSingle_l$ (with $\livegroupSingle_1 = \livegroupSingle$).
    Suppose $\run$ does not satisfy the live-group template $\livegroupSingle_i$ for all $i\leq l'$ and satisfies $\livegroupSingle_i$ for all $l'<i\leq l$. 
    Note that $l' > 1$ as $\run$ does not satisfy $\livegroupSingle = \livegroupSingle_1$. 
    Then, for every history $\hist q$ of $\run$ after which $\run$ does not contain any co-live edge and
    for every action $a'\in A' = \{a'\in A(q)\mid \edge{q}{a'}\not\in\bigcup_{i=1}^{l'}\livegroupSingle_i\}$, we have:
    \[ \shield{\policy}(\hist q, a') \leq \frac{\policy(\hist q, a) + \gamma \cdot \sum_{i=l'}^l \counter{\livegroupSingle_i}(\hist q)}{1 + \gamma \cdot \sum_{i=1}^l \counter{\livegroupSingle_i}(\hist q)}. \]
    Furthermore, there exists a history $\hist'$ of $\run$ after which $\run$ (does not use any co-live edge and) visits $q$ infinitely many times but never samples an action from any of the live-groups $\livegroupSingle_1,\ldots,\livegroupSingle_{l'}$.
    From that point on, the counter $\counter{\livegroupSingle_{i}}$ for $i\leq l'$ is incremented unboudedly, whereas the counter $\counter{\livegroupSingle_{i}}$ for $i\geq l'$ is reset to zero infinitely many times.
    Consequently, there exists a history $\hist q$ of $\run$ (that is an extension of $\hist'$) such that for every $a'\in A'$, we have:
    \[\frac{\policy(\hist q, a') + \gamma \cdot \sum_{i=l'}^l \counter{\livegroupSingle_i}(\hist q)}{1 + \gamma \cdot \sum_{i=1}^l \counter{\livegroupSingle_i}(\hist q)} < \thres. \]
    By construction, $\shield{\policy}(\hist q, a') = 0$ for every $a'\in A'$, and hence, $\shield{\policy}$ has to sample an action from $A\setminus A'$ at history $\hist q$. 
    As all actions in $A\setminus A'$ are from the live-groups $\bigcup_{i=1}^{l'}\livegroupSingle_i$, this contradicts the assumption that $\run$ does not sample from $\bigcup_{i=1}^{l'}$ after $\hist'$.
    Thus, $\run$ satisfies the live-group template $\livegroup$.   
    Therefore, $\shield{\policy}$ follows the template $\template_\psure = (\safegroup,\colivegroup,\livegroup)$.

    Finally, if $\template_\psure = \parityTemp_\psure(M,\Phi)$, then, by the properties of winning strategy templates~\cite{AnandNS23}, it holds that every $\shield{\policy}$-run from the winning region of $\Phi$ satisfies $\Phi$ surely/almost surely (depending on $\psure$).
\end{proof}

\subsection{Minimal Interference of \Stars{}}
We now restate the two minimal interference properties of \Stars{} as in Theorem~2 and Theorem~3 of the main text and then prove them.
\begin{restatable}{theorem}{restateMinimalInterference}\label{thm:minimalinterference}
    Given the premises of \cref{corollary:shieldedPolicy} with $\shield{\policy} = \policy|^{\template_\psure,\thres}_{\gamma}$,   
    for any $\varepsilon > 0$ and for all length $l \in \mathbb{N}$, there exist  parameters $\gamma,\thres > 0$ such that for all histories $\hist$ of length $l$ with $\hist \models \pref(\Phi)$, it holds that
    $\Pr_{\shield{\policy}}(\hist) > \Pr_{\policy}(\hist) -  \varepsilon.$
\end{restatable}
\begin{proof}
    Let us fix an $\varepsilon>0$ and a length $l$.
    Let $\hist = q_0a_0\ldots q_l\in \fruns^M$ be a history such that $\hist \models \pref(\Phi)$.
    Clearly, for edges $\edge{q}{a} \in \safegroup$, $\edge{q}{a}$ cannot appear in $\hist$ as then any extension of $\hist$ will not satisfy $\Phi$. 

    First, let $\policy(\hist[0;i], a_{i}) = x_i$ and $\Pr(q_{i+1}| q_{i}, a_{i}) = y_i$.
    Then it holds that
        $\Pr_{\policy}(\hist)
        = \prod_{0 \leq i \leq l-1} x_iy_i$.
    Note that whenever $x_i = 0$ for any $i\leq l$, then $\Pr_{\policy}(\hist) = 0$ and hence, $\Pr_{\shield{\policy}}(\hist) > \Pr_{\policy}(\hist) -  \varepsilon$ trivially holds.
    Let us therefore now consider the case where $x_i > 0$ for all $i\leq l$.

    Now, for every $i\leq l$, we have $\sum_{e\in\colivegroup} \counter{e}(\hist[0;i]) \leq l$ and $\sum_{\livegroupSingle\in\livegroup} \counter{\livegroupSingle}(\hist[0;i]) \leq \abs{\livegroup}l$.
    Hence, if $\shield{\policy}(\hist[0;i],a_{i}) = x_i'$, 
    by taking small enough $\thres$,
    it holds that
    \[x_i' \geq \frac{x_i-l\cdot\gamma}{1+\abs{\livegroup}l\cdot\gamma}.\]
    Thus, we have:
    \[\Pr_{\shield{\policy}}(\hist) 
        = \prod_{0 \leq i \leq l-1} x_i'y_i 
        \geq \prod_{0 \leq i \leq l-1} \frac{x_i-l\cdot\gamma}{1+\abs{\livegroup}l\cdot\gamma}\cdot y_i.\]
    For $x = \min_{0 \leq i \leq l-1} x_i$, we have:
    \begin{align*}
        \Pr_{\policy}(\hist) - \Pr_{\shield{\policy}}(\hist)
        &\leq \prod_{0 \leq i \leq l-1} x_iy_i - \prod_{0 \leq i \leq l-1} \frac{x_i-l\cdot\gamma}{1+\abs{\livegroup}l\cdot\gamma}\cdot y_i\\
        &= \Pr_{\policy}(\hist) \cdot \left(1-\prod_{0 \leq i \leq l-1} \frac{1-\frac{l\cdot\gamma}{x_i}}{1+\abs{\livegroup}l\cdot\gamma}\right)\\
        &\leq \Pr_{\policy}(\hist) \cdot \left(1- \left(\frac{1-\frac{l\cdot\gamma}{x}}{1+1\abs{\livegroup}l\cdot\gamma}\right)^l\right)
    \end{align*}
    It now follows that by fixing an appropriate value for $\gamma$, one can bound the above expression by $\varepsilon$. 
    As the above expression is independent of the choice of history $\hist$, we can conclude that for every history $\hist$ of length $l$ such that $\hist \models \pref(\Phi)$, it holds that: $\Pr_{\shield{\policy}}(\hist) > \Pr_{\policy}(\hist) - \varepsilon$.
\end{proof}

\begin{restatable}{theorem}{restateMinimalInterferenceWeighted}\label{thm:minimalinterferenceWeighted}
Given the premises of \cref{corollary:shieldedPolicy} with $\shield{\policy} = \policy|^{\template_\psure,\thres}_{\gamma}$ such that $\template_\psure = (\emptyset,\emptyset,\livegroup)$, and a cost function $\cost:\fruns^M \rightarrow [0,W]$, for any $\varepsilon > 0$, there exist parameters $\gamma,\thres > 0$ such that the following holds: $\exp_{\rho\sim\shield{\policy}} \cost(\rho,\policy,\shield{\policy}) < \varepsilon$.
\end{restatable}
\begin{proof}
    First let us show that the tuple of counter values can be bounded based on the parameters $\thres$ and $\gamma$. As $\template_\psure = (\emptyset,\emptyset,\livegroup)$, we only have counters for the live-groups in $\livegroup$. Hence, the following claim formalizes the bound on the counter values.
    
    \begin{claim}\label{claim:boundedcounter}
        For parameters $K_{\thres,\gamma} = \max\left(1,\frac{1/\thres-1}{\gamma}\right)$, every counter value $\counter{\livegroupSingle}(\hist) \leq K_{\thres,\gamma}$ for every history $\hist$ and live-group $\livegroupSingle\in\livegroup$.
    \end{claim}
    \begin{claimproof}
    By the construction of $\parityTemp_\psure$, as in~\cite[Alg. 3]{AnandNS23}, every state $q$ can be a source state for at most one live-group $\livegroupSingle\in\livegroup$. 
    Hence, for a history $\hist$ ending in such a state $q$, for every action $a\in A(q)\setminus\livegroupSingle$, we have 
    \[\shield{\policy}(\hist, a) \leq \frac{\policy(\hist,a)}{1+\counter{\livegroupSingle}(\hist)\cdot\gamma} \leq \frac{1}{1+\counter{\livegroupSingle}(\hist)\cdot\gamma}.\]
    Hence, whenever $\counter{\livegroupSingle}(\hist)\geq K_{\thres,\gamma}$, we have $\shield{\policy}(\hist, a) \leq \thres$ and hence, by definition,
    $\shield{\policy}(\hist, a) = 0$.
    Therefore, for $\counter{\livegroupSingle}(\hist) \geq K_{\thres,\gamma}$, probability of sampling an action not in $\livegroupSingle$ becomes zero and hence, an action from $\livegroupSingle$ will be sampled and hence, $\counter{\livegroupSingle}(\hist)$ will be reset to zero.
    So, the counter value $\counter{\livegroupSingle}(\hist) \leq K_{\thres,\gamma}$ for every history $\hist$.
    \end{claimproof}

    Now, w.l.o.g, let us assume that $\policy$ is a stationary policy in $M$ as otherwise we can take the product of $M$ with the memory set of $\policy$ to ensure that $\policy$ is stationary in the product MDP.
    Furthermore, $\shield{\policy}$ is a stationary policy w.r.t.\ the extended MDP $M' = \langle Q', A, \Delta', \qinit'\rangle$
    obtained by taking product of the MDP with the tuples of counter values.
    Since $\policy$ is a stationary stochastic policy in $M$, it is also a stationary stochastic policy in $M'$.

    Now, given an $\varepsilon>0$ and a cost function $\cost: \fruns^M \rightarrow [0,W]$, it holds that $\cost(\hist) \leq W$ for every history $\hist$.
    Hence, the following holds for stationary distribution $d_{\shield{\policy}}$ of $\shield{\policy}$ in $M'$:
    \begin{align*}
        \exp_{\rho\sim\shield{\policy}}\cost(\rho,\policy,\shield{\policy}) &= \exp_{\rho\sim\shield{\policy}} \left[\limsup_{l\to\infty} \frac{1}{l}\sum_{i=0}^{l-1} \cost(\rho[0;i])\cdot \distance(\policy(\rho[0;i]),\shield{\policy}(\rho[0;i]))\right] \\
        &= W\cdot \exp_{\rho\sim\shield{\policy}} \left[\limsup_{l\to\infty} \frac{1}{l}\sum_{i=0}^{l-1} \distance(\policy(\rho[0;i]),\shield{\policy}(\rho[0;i]))\right]\\
        &= W\cdot \limsup_{l\to\infty} \frac{1}{l}\cdot  \exp_{\rho\sim\shield{\policy};\abs{\rho} = l} \left[\sum_{i=0}^{l-1} \distance(\policy(\rho[0;i]),\shield{\policy}(\rho[0;i]))\right]\\
        &= W\cdot \limsup_{l\to\infty} \frac{1}{l}\cdot\sum_{i=0}^{l-1} \exp_{(q,C)\sim d_{\shield{\policy}}} \left[\distance(\policy(q),\shield{\policy}((q,C)))\right]\\
        &= W\cdot \exp_{(q,C)\sim d_{\shield{\policy}}} \left[\distance(\policy(q),\shield{\policy}((q,C)))\right].
    \end{align*}
    Evaluating the total variation distance $\distance(\policy(q),\shield{\policy}((q,C)))$ gives us the following:
    \begin{equation}\label{eq:cost_tvdist}
        \exp_{\rho\sim\shield{\policy}}\cost(\rho,\policy,\shield{\policy}) = \dfrac12W\cdot \exp_{(q,C)\sim d_{\shield{\policy}}} \sum_{a\in A(q)} \abs{\policy(q,a) - \shield{\policy}((q,C),a)}.
    \end{equation}
    Then, as $\template_\psure = (\emptyset,\emptyset,\livegroup)$, 
    if $\policy''((q,C),a)$ is the distribution obtained 
    before bounding the probabilities by $\thres$, then the following holds: 
    \begin{align*}
        \policy''((q,C),a) \geq \frac{\policy(q,a)}{1+\abs{C}\gamma} \quad\text{and}\quad \policy''((q,C),a) \leq \frac{\policy(q,a)+\abs{C}\gamma}{1+\abs{C}\gamma},
    \end{align*}
    where $\abs{C}$ denotes the sum of the counters in $C$.

    \noindent
    This means, after bounding the probabilities by $\thres$, if $\policy''((q,C),a) < \thres$, then $\shield{\policy}((q,C),a) = 0 = \policy''((q,C),a)-\thres$, and if some probabilities are bounded by $\thres$, then $\shield{\policy}((q,C),a) \leq \policy''((q,C),a) + \abs{A(q)}\cdot\thres$.
    Hence, it holds that:
    \begin{align*}
        \shield{\policy}((q,C),a)\geq \policy''((q,C),a)-\thres &\quad\text{and}\quad \shield{\policy}((q,C),a) \leq \policy''((q,C),a) + \abs{A(q)}\cdot\thres\\
        \Rightarrow\shield{\policy}((q,C),a)\geq \frac{\policy(q,a)}{1+\abs{C}\gamma}-\thres &\quad\text{and}\quad \shield{\policy}((q,C),a) \leq \frac{\policy(q,a)+\abs{C}\gamma}{1+\abs{C}\gamma} + \abs{A(q)}\cdot\thres\\
        \Rightarrow\shield{\policy}((q,C),a)\geq \frac{\policy(q,a)}{1+\abs{C}\gamma}-\abs{A(q)}\cdot\thres &\quad\text{and}\quad \shield{\policy}((q,C),a) \leq \frac{\policy(q,a)+\abs{C}\gamma}{1+\abs{C}\gamma} + \abs{A(q)}\cdot\thres.
    \end{align*}
    Therefore, we have:
    \begin{align*}
        \abs{\policy(q,a) - \shield{\policy}((q,C),a)} 
        &\leq \max\left\{\policy(q,a) - \shield{\policy}((q,C),a),\quad \shield{\policy}((q,C),a) - \policy(q,a)\right\}\\
        &\leq \max\Big\{\policy(q,a)\left(\frac{\abs{C}\gamma}{1+\abs{C}\gamma}\right), \frac{\abs{C}\gamma}{1+\abs{C}\gamma}(1-\policy(q,a))\Big\}+\abs{A(q)}\cdot\thres\\
        &\leq \abs{C}\gamma + \abs{A}\thres.
    \end{align*}
    Hence, the expected cost in \eqref{eq:cost_tvdist} can be bounded as follows:
    \begin{align*}
        \exp_{\rho\sim\shield{\policy}}\cost(\rho,\policy,\shield{\policy}) 
        &\leq \dfrac12W\cdot \exp_{(q,C)\sim d_{\shield{\policy}}} \sum_{a\in A(q)} (\abs{C}\gamma + \abs{A}\cdot\thres)\\
        &\leq \dfrac12W\cdot \sum_{a\in A} \exp_{(q,C)\sim d_{\shield{\policy}}} [\abs{C}\gamma + \abs{A}\cdot\thres]\\
        &\leq \dfrac12W\cdot \abs{A}^2\cdot\thres + W\abs{A} \gamma \cdot \exp_{(q,C)\sim d_{\shield{\policy}}}[\abs{C}]\\
    \end{align*}
    From the construction of the shielded policy $\shield{\policy}$, we know that $\abs{C}$ increases only if the current state $q$ is a source state of some live edge in a live-group $\livegroupSingle \in \livegroup$ and none of the live edges in $\livegroupSingle$ are sampled.
    From such a state $(q,C)$ with high enough counter value $\counter{\livegroupSingle}(C)$, the probability of sampling an action that is not in $\livegroupSingle$ can be expressed in terms of $\policy(q,A(q)\cap \livegroupSingle) = \sum_{a\in A(q)\cap \livegroupSingle}\policy(q,a)$ as follows:
    \begin{align*}
        \Pr[(q,C)\rightarrow C+1] &= 1-\sum_{a\in A(q)\cap \livegroupSingle} \shield{\policy}((q,C),a)\\
        &\leq 1 - \frac{\policy(q,A(q)\cap \livegroupSingle)+\counter{\livegroupSingle}(C)}{1+\counter{\livegroupSingle}(C)\cdot \gamma}\\
        &\leq \frac{1-\policy(q,A(q)\cap \livegroupSingle)}{1+\counter{\livegroupSingle}(C)\cdot \gamma} \\
        &\leq 1-\policy(q,A(q)\cap \livegroupSingle).
    \end{align*}
    If $\policy(q,A(q)\cap \livegroupSingle) = 1$, then $\Pr[(q,C)\rightarrow C+1] = 0$. 
    Hence, if $\Pr[(q,C)\rightarrow C+1] > 0$, then the probability $\Pr[(q,C)\rightarrow C+1]$ can be bounded by 
    \[ m = \max \{1 - \policy(q,A(q)\cap \livegroupSingle) \mid q\in Q, \livegroupSingle\in\livegroup, A(q)\cap \livegroupSingle\neq \emptyset, \policy(q,A(q)\cap \livegroupSingle) < 1\}.\]
    Note that $m>0$ as $\policy$ is a stochastic policy and $m<1$ by the above construction.
    Hence, the probability that the counter sum $\abs{C}$ increases can be bounded by $m$. Therefore, the expected value of $\abs{C}$ can be bounded as follows:
    \begin{align*}
        \exp_{(q,C)\sim d_{\shield{\policy}}}[\abs{C}] 
        &\leq \sum_{i=0}^{\infty} i \cdot m^i
        = \frac{m}{(1-m)^2}
        \leq \frac{1}{(1-m)^2}.
    \end{align*}
    Therefore, the expected cost in \eqref{eq:cost_tvdist} can be bounded as follows:
    \begin{align*}
        \exp_{\rho\sim\shield{\policy}}\cost(\rho,\policy,\shield{\policy}) 
        \leq \dfrac12W\cdot \abs{A}^2\cdot\thres + W\abs{A} \gamma \cdot \frac{1}{(1-m)^2}.
    \end{align*}
    By fixing $\thres$ and $\gamma$ appropriately, we can ensure that the above expression is less than $\varepsilon$.
\end{proof}

\section{Maintaining Optimal Rewards while Shielding}
In this section, we formalize the problem and results discussed in Section 3.5 of the main text in details, which is to maintain optimal rewards while shielding a policy $\policy$.

\label{sec:optimalshield}

We restate the two rewards we consider in section 3.5 of the main text. 
Given a rewardful MDP $(M,r)$ with initial state $\qinit$ and a policy $\policy$, we define the \emph{discounted reward} and the \emph{average reward} via 
\begin{subequations}
    \begin{align}
    &\textstyle
     \discreward{\lambda}{\qinit}{\policy}:=\lim_{N \to\infty} \exp^{\qinit}_{\policy} \left( \sum_{0 \leq i \leq N} \lambda^i r(X_i,Y_i) \right)~\text{with }\lambda \in [0, 1],~\text{and}\label{eq:discounted}\\
        &\textstyle
     \avgreward{\qinit}{\policy}:=\limsup_{N \to\infty} \frac{1}{N} \exp^{\qinit}_{\policy} \left( \sum_{0 \leq i \leq N} r(X_i,Y_i) \right).\label{equ:average}
    \end{align}
    \end{subequations}
For any reward function, we define the \emph{optimal reward} to be the maximum reward achievable by a policy $\policy$. We call every policy that achieves the optimal reward an \emph{optimal policy}. A policy $\policy$ is \emph{$\varepsilon$-optimal} if it achieves a reward that is at least $\varepsilon$ less than the optimal reward.

Now, we show the optimality of shielded policies as stated in Theorem~4 of the main text by considering the discounted and average rewards separately as in the following two subsections.
\setcounter{theorem}{4}

\subsection{Discounted Rewards}\label{sec:discounted}
When policies are trained to maximize a discounted reward, as formalized in~\eqref{eq:discounted}, the impact of obtained rewards decreases with time. Therefore, the optimal reward achievable by a policy over a given MDP 
significantly depends on the bounded (initial) executions possible over this MDP. As \cref{thm:minimalinterference} shows that the probability of observing a bounded execution remains close to its probability under the original policy, the $\varepsilon$-optimality of a shielded optimal policy can be obtained as a direct consequence of \cref{thm:minimalinterference} when discounted rewards are used during training.
In particular, the discounted reward of the shielded policy remains close to the discounted reward of the original policy as formalized below.

\begin{restatable}{theorem}{restateDiscountedOptimal}\label{thm:discountedoptimalshield}
    Given the premises of \cref{corollary:shieldedPolicy} with $\shield{\policy} = \policy|^{\template_\psure,\thres}_{\gamma}$ and $\win^{\psure}_\Phi = Q$, for every $\varepsilon > 0$, there exist parameters $\thres,\gamma > 0$ such that the following holds:
    $\discreward{\lambda}{\qinit}{\shield{\policy}} > \discreward{\lambda}{\qinit}{\policy} - \varepsilon$.
\end{restatable}
\begin{proof}
    Let us first compute a bound on the length of the runs that are significant to get an $\varepsilon$-optimal reward.
    Let $r_{max}$ be the maximum reward in the MDP $M$, and let $ \discreward{\lambda,k}{\qinit}{\policy'}$ be the expected bounded discounted reward of policy $\policy'$ for the first $k$ steps:
    \[ \discreward{\lambda,k}{\qinit}{\policy'} = \sum_{1 \leq i \leq k} \lambda^i \exp^{\qinit}_{\policy}(r(X_i,Y_i))\]
    Now, let's choose $B$ such that $\lambda^B \cdot r_{max} < (1-\lambda)\cdot\varepsilon/2$.
    Note that such a $B$ exists as $\lambda^B \rightarrow 0$ when $B \rightarrow \infty$.
    This gives us the following for any policy~$\policy'$:
    \begin{align*}
        \discreward{\lambda}{\qinit}{\policy'}
        &= \lim_{N \mapsto\infty} \sum_{0 \leq i \leq N-1} \lambda^i \exp^{\qinit}_{\policy}(r(X_i,Y_i))\\
        &= \discreward{\lambda,B}{\qinit}{\policy'} + \lim_{N \mapsto\infty} \sum_{B+1 \leq i \leq N} \lambda^i \exp^{\qinit}_{\policy}(r(X_i,Y_i))\\
        &\leq  \discreward{\lambda,B}{\qinit}{\policy'} + \lim_{N \mapsto\infty} \sum_{B+1 \leq i \leq N} \lambda^i \cdot r_{max}\\
        &=  \discreward{\lambda,B}{\qinit}{\policy'} + \frac{\lambda^B \cdot r_{max}}{1-\lambda}\\
        &<  \discreward{\lambda,B}{\qinit}{\policy'} + \frac{\varepsilon}{2}
    \end{align*}

    Now, let $r(\hist) = \sum_{0 \leq i \leq \abs{\hist}-1} \lambda^i r(X_i,Y_i)
    $ be the reward of a run $\hist$. Then, we can rewrite $\discreward{\lambda,k}{\qinit}{\policy'}$ in terms of the expected reward of $k$-length runs as follows:
    \[\discreward{\lambda,k}{\qinit}{\policy'} = \sum_{\hist \in \fruns^M; \abs{\hist} = k} \exp^{\qinit}_{\policy}(\prob{\hist}) \cdot r(\hist)\]
    Hence, the difference between bounded discounted reward for optimal policy $\policy$ and shielded policy $\shield{\policy}$ is the following:
    \begin{align*}
        \discreward{\lambda,B}{\qinit}{\policy} - \discreward{\lambda,B}{\qinit}{\shield{\policy}}
        &= \sum_{\hist \in \fruns^M; \abs{\hist} = B} (\exp^{\qinit}_{\policy}(\prob{\hist}) - \exp^{\qinit}_{\shield{\policy}}(\prob{\hist})) \cdot r(\hist)\\
        &=\sum_{\hist \in \fruns^M; \abs{\hist} = B} (\Pr_{\shield{\strat}}(\hist) - \Pr_{\strat}(\hist)) \cdot r(\hist)\\
        &\leq \abs{Q}^B\cdot (\Pr_{\shield{\strat}}(\hist) - \Pr_{\strat}(\hist)) \cdot r(\hist).
    \end{align*}
    As $\win^{\psure}_\Phi = Q$, every history $\hist\models \pref(\Phi)$.
    Hence, using \cref{thm:minimalinterference}, for length $B$, there exists parameters $\thres,\gamma>0$ such that we can bound the above term by $\varepsilon/2$. Using the property of bound $B$, this gives us:
    \begin{align*}
        &\discreward{\lambda,B}{\qinit}{\policy} - \discreward{\lambda,B}{\qinit}{\shield{\policy}} < \varepsilon/2\\
        \implies &\discreward{\lambda,B}{\qinit}{\shield{\policy}} + \varepsilon/2 > \discreward{\lambda,B}{\qinit}{\policy} > \discreward{\lambda}{\qinit}{\policy} - \varepsilon/2\\
        \implies &\discreward{\lambda,B}{\qinit}{\shield{\policy}} > \discreward{\lambda}{\qinit}{\policy} -\varepsilon.
    \end{align*}
    As $\discreward{\lambda}{\qinit}{\shield{\policy}} \geq \discreward{\lambda,B}{\qinit}{\shield{\policy}}$, we have that $\discreward{\lambda}{\qinit}{\shield{\policy}} > \discreward{\lambda}{\qinit}{\policy} -\varepsilon$.
\end{proof}

\subsection{Average Rewards}\label{sec:average}
We now consider the scenario that policies where trained to maximize the average reward, as formalized in \eqref{equ:average}. In contrast to policies which optimize a discounted reward, optimal average reward policies do not put special emphasis on bounded (initial) executions. On the contrary, the optimal average reward is equivalently impacted by rewards collected over the entire (infinite) length of runs compliant with the policy. This implies, that a policy can only satisfy an $\omega$-regular property (almost) surely \emph{and} optimize the average reward, if it can \enquote{switch} between their satisfaction by alternating infinitely often between finite intervals which satisfy either one. This, however, is only possible if the underlying MDP is `nice' enough to allow for this alternation. In particular, to retain the modularity of shielding with \Stars, we demand to be able to shield a policy over an MDP \enquote{blindly}, i.e.\ without assuming access to the reward-structure over $M$ or knowledge of the actual policy. We therefore demand, in addition to $\win^{\psure}_\Phi = Q$ assumed in \cref{thm:discountedoptimalshield} (further discussed in Rem.~3 in the main text), that $\win^\psure_\Phi$ is a strongly connected component (SCC).
With this assumption, we achieve $\varepsilon$-optimality of the average reward for the special class of Büchi objectives as a consequence of \cref{thm:minimalinterferenceWeighted}.

\begin{restatable}{theorem}{restateAvgOptimal}\label{thm:optimalshield}
    Given the premises of \cref{corollary:shieldedPolicy} with $\shield{\policy} = \policy|^{\template_\psure,\thres}_{\gamma}$, B\"uchi objective $\Phi$, and SCC $\win^{\psure}_\Phi = Q$, for every $\varepsilon > 0$, there exist $\thres,\gamma > 0$ such that 
    $\avgreward{\qinit}{\shield{\policy}} > \avgreward{\qinit}{\policy} - \varepsilon$ holds.
\end{restatable}
\begin{proof}
    First, note that $\template = (\safegroup,\colivegroup,\livegroup)$ such that $\safegroup = \colivegroup = \emptyset$ as $\win^{\psure}_\Phi = Q$ and $\Phi$ is a B\"uchi objective~\cite{AnandNS23}.
    By \cref{claim:boundedcounter}, we know that each counter value is bounded by $K_{\thres,\gamma}$.
    Now, as in the proof of \cref{thm:minimalinterferenceWeighted}, w.l.o.g, let us assume that $\policy$ is a stationary policy in $M$ and hence, one can show that both $\shield{\policy}$ and $\policy$ are stationary policies w.r.t.\ the extended MDP $M' = \langle Q', A, \Delta', \qinit'\rangle$ obtained by taking product of the MDP with the tuples of bounded counter values.
    Furthermore, as $Q$ is an SCC, using an extension of the well-known policy difference lemma~\cite{Langford} for average rewards (see~\cite{Zhang0010R21,Puterman} for more details), the difference in average rewards of $\policy$ and $\shield{\policy}$ can be expressed as\footnote{We use $x\sim \distr$ to denote that $x$ is sampled from distribution $\distr$.}: 
    \begin{align}\label{eq:avgrewarddiff}
        \Delta\avgreward{}{} = \avgreward{\qinit}{\policy} - \avgreward{\qinit}{\shield{\policy}} &= \exp_{\substack{(q,C)\sim d_{\shield{\policy}} \\ a\sim\shield{\policy}}} \Big[\biasstate((q,C)) - \biasact((q,C),a) \Big],
    \end{align}
    where $d_{\shield{\policy}}$ is the stationary distribution of $\shield{\policy}$ in $M'$; $\biasact$ and $\biasstate$ are the action-bias and state-bias functions of $\policy$ defined as:
    \begin{align*}
        \biasact((q,C),a) = \biasact(q,a) 
        &= \exp_{\run\sim\policy} \Big[\sum_{i=0}^\infty r(\run[i]) - \avgreward{\qinit}{\policy} \given \run[0] = (q,a)\Big],\\
        \biasstate((q,C)) = \biasstate(q) 
        &= \exp_{\run\sim\policy} \Big[\sum_{i=0}^\infty r(\run[i]) - \avgreward{\qinit}{\policy} \given \run[0] \in q \times A(q)\Big].
    \end{align*}
    Then, \eqref{eq:avgrewarddiff} can be rewritten as:
    \begin{align*}
        \Delta\avgreward{}{}
        &= \exp_{(q,C)\sim d_{\shield{\policy}}} \sum_{a\in A(q)} \shield{\policy}((q,C),a) \cdot \Big[\biasstate((q,C)) - \biasact((q,C),a) \Big]\\
        &= \exp_{(q,C)\sim d_{\shield{\policy}}} \sum_{a\in A(q)} \shield{\policy}((q,C),a) \cdot \Big[\sum_{a'\in A(q)} \policy(q,a') \biasact(q,a') - \biasact(q,a) \Big]\\
        &= \exp_{(q,C)\sim d_{\shield{\policy}}} \sum_{a'\in A(q)} \policy(q,a') \biasact(q,a') - \sum_{a\in A(q)} \shield{\policy}((q,C),a) \cdot \biasact(q,a)\\
        &= \exp_{(q,C)\sim d_{\shield{\policy}}} \sum_{a\in A(q)} (\policy(q,a) - \shield{\policy}((q,C),a)) \cdot \biasact(q,a).
    \end{align*}
    As it is knows that $\abs{\biasact}$ is bounded by some $B$ for such policy $\policy$~\cite{Puterman}, we have the following:
    \begin{align*}
        \Delta\avgreward{}{}
        &\leq B \cdot \exp_{(q,C)\sim d_{\shield{\policy}}} \sum_{a\in A(q)} \abs{\policy(q,a) - \shield{\policy}((q,C),a)}.
    \end{align*}
    As the above expression is similar to the expected cost in \eqref{eq:cost_tvdist}, using similar arguments as in the proof of \cref{thm:minimalinterferenceWeighted}, using appropriate values for $\thres$ and $\gamma$, we can ensure that the above expression is less than $\varepsilon$.
\end{proof}

Unfortunately, the $\varepsilon$-optimality of shielded policies established for Büchi objectives in \cref{thm:optimalshield} does not directly generalize to parity conditions. To `blindly' shield for the latter, we require the following definition.

\begin{definition}
 Let $M$ be an MDP and $c: Q \rightarrow [0;d]$ be a coloring function induced by a parity condition $\Phi$ over $M$. Let $\widetilde{Q}\subseteq Q$ be an SCC. We say $\tilde{Q}$ is $\psure$-\textsf{good} w.r.t.\ $\Phi$ if $q \in \{\widetilde{Q}\mid c(q) \text{ is odd}\}$ implies $q \in \win^{\psure}_{\varphi_\textsf{good}}$, where
$\varphi_\textsf{good}:= \{\rho \in \runs^M|_{\widetilde{Q}} \mid \exists n \geq 0: c(\rho[n]) > c(q) \text{ and is even}\},$
i.e., from every odd state in $\widetilde{Q}$, the system player can surely/almost surely visit a higher even state in $\widetilde{Q}$.
\end{definition}

It is known that the optimal average reward achievable over an MDP $M$ while (almost) surely satisfying a parity condition $\Phi$ reduces to (i) finding all $\psure$-\textsf{good} SCC's of $M$ w.r.t.\ $\Phi$, (ii) computing the optimal average reward of each $\psure$-\textsf{good} SCC, and (iii) enforcing reaching a $\psure$-\textsf{good} SCC with the highest achievable optimal average reward (see \cite{AlmagorKV16} for details). 
As a consequence, the results of \cref{thm:optimalshield} carry over to parity objectives if $Q$ is a $\psure$-\textsf{good} SCC. 

In addition, within $\psure$-\textsf{good} SCCs, the full expressive power of strategy templates is not required—adding co-live templates $D$ does not yield additional winning strategies. This is because for any state $q$, if $q'$ is a state with the maximal color reachable (almost) surely from $q$, then $c(q')$ must be even; otherwise, by definition of a $\psure$-\textsf{good} SCC, a higher even-colored state would be (almost) surely reachable. Consequently, winning strategy templates over $\psure$-\textsf{good} SCCs only require live-groups to reach maximum (even) color states from each state. It follows that $\parityTemp_\psure(M|_{\widetilde{Q}}, \Phi)$ contains only live-groups and unsafe edges, as in the Büchi case. 
This is formalized in the following lemma.
\begin{lemma}\label{lemma:templateGoodSCC}
    Let $M$ be an MDP and $\Phi = \textsc{Parity}[c]$ be a parity objective such that $G^M = (Q,E)$ is the corresponding $\psure$-game graph.
    If $Q$ is a $\psure$-good SCC, then the strategy template $\template_\psure = \parityTemp_\psure(M,\Phi)$ is such that $\template_\psure = (\safegroup,\colivegroup,\livegroup)$ with $\safegroup = \colivegroup = \emptyset$.
\end{lemma}
\begin{proof}
    First, let us consider the case when $\psure = \sure$.
    Let $\template_\sure = (\safegroup,\colivegroup,\livegroup)$ be the strategy template obtained by the procedure $\parityTemp_\sure(M,\Phi)$. 
    Note that as $Q$ is a $\sure$-good SCC, by results in~\cite{AlmagorKV16}, $\win^{\sure}_\Phi = Q$. As the procedure $\parityTemp_\sure$ marks all edges going out of $\win^{\sure}_\Phi$ as unsafe, it follows that $\safegroup = \emptyset$.
    We only need to show that $\colivegroup = \emptyset$.

    Let $\Phi = \textsc{Parity}[c]$ be a parity objective with coloring function $c: Q \rightarrow [0;d]$. Let $P_i = \{q \in Q \mid c(q) = i\}$ be the set of states with color $i$.
    Let us first show that the maximal color $d$ can only be even.
    Suppose $d$ is odd, then there exists a state $q$ with $c(q) = d$.
    As $Q$ is a $\sure$-good SCC, it holds that $q \in \win^{\sure}_\varphi$, where
    \[\varphi = \{\rho \in \runs^M \mid \exists n \geq 0: c(\rho[n]) > d \text{ and is even}\}.\]
    However, as $d$ is the maximum color, there are no states with $c(q) > d$, and hence, $\varphi = \emptyset$. So, $\win^{\sure}_\varphi = \emptyset$, which is a contradiction to the assumption that $q \in \win^{\sure}_\varphi$.

    Now, let us show that $\colivegroup = \emptyset$ using induction on $d$, i.e., the maximal color in $c$.
    For the base case, if $d = 0$, then the statement holds trivially.
    
    Suppose the statement holds for $d = 2k$ for some $k \geq 0$.
    Now, let us consider the case when $d = 2k+2$.
    Consider the procedure $\parityTemp$ given in \cite[Algorithm 3]{AnandNS23} that computes the strategy template $\template_\sure$.
    As $d$ is even, the procedure $\parityTemp$ starts by computing $A = \mathtt{attr}^0(P_d)$ (in line 14) which is the set of states from which system player can force the play to visit a state with color $d$.
    If $A = Q$, then the procedure terminates (in line 15) and returns the output of a procedure $\textsc{ReachTemplate}$~(provided in \cite[Algorithm 1]{AnandNS23}), which only returns a live group template. Hence, $\colivegroup = \emptyset$.
    Suppose $A \neq Q$, then the procedure $\parityTemp$ is called recursively on the subgame $G'$ obtained by removing the states in $A$ (in line 17).
    As $A = \mathtt{attr}^0(P_d) \supseteq P_d$, the maximal color in $G'$ is at most $d-1$.
    If there is a state $q$ in $G'$ with color $d-1$ (which is odd), then as $Q$ is a $\sure$-good SCC, it holds that $q \in \win^{\sure}_{\varphi'}$, where
    \[\varphi' = \{\rho \in \runs^M \mid \exists n \geq 0: c(\rho[n]) > d-1 \text{ and is even}\}.\]
    Hence, system player can force the run from $q$ to visit a state with color $d$, and hence, $q\in\mathtt{attr}^0(P_d) = A$. This contradicts the assumption that $q$ is a node in $G'$.
    Therefore, there is no state in $A$ with color $d-1$, and hence, the maximal color in $G'$ is at most $d-2$. By induction hypothesis, the procedure $\parityTemp$ returns a template with $\colivegroup = \emptyset$.
    
    For the case when $\psure = \asure$, the procedure $\parityTemp$ uses Algorithm 8 provided in \cite{phalakarn2024winningstrategytemplatesstochastic}.
    This procedure first converts the $1\thalf$-player game graph $G^M$ to a $2$-player game graph $G'$ and then uses the procedure $\parityTemp$ provided in \cite[Algorithm 3]{AnandNS23} to compute the strategy template.
    Using similar arguments as above, we can show that the procedure $\parityTemp$ provided in \cite[Algorithm 3]{AnandNS23} for the $2$-player game graph $G'$ returns a template with $\colivegroup = \emptyset$.
\end{proof}

With the above lemma, the following result is a direct corollary of \cref{thm:optimalshield}.

\begin{corollary}\label{cor:optimalshieldparity}
    Given the premises of \cref{corollary:shieldedPolicy} with $\shield{\policy} = \policy|^{\template_\psure,\thres}_{\gamma}$, parity objective $\Phi$, and $\psure$-\textsf{good} SCC $\win^{\psure}_\Phi = Q$, for every $\varepsilon > 0$, there exist $\thres,\gamma > 0$ such that 
    $\avgreward{\qinit}{\shield{\policy}} > \avgreward{\qinit}{\policy} - \varepsilon$ holds.
\end{corollary}

\begin{remark}
We note that Rem.~3 in the main text directly transfers to \cref{thm:optimalshield} and \cref{cor:optimalshieldparity}. That is, we can restrict the training of an optimal average reward policy to a ($\psure$-\textsf{good}-)SCCs and shield the resulting policy with \Stars therein.   
If we would like to maximize the state space over which our shield is applicable,
we can only restrict learning to $\win^{\psure}_\Phi$, as in the discounted reward case. We can then compute separate \Stars for every $\psure$-\textsf{good} SCC of a given MDP $M$ and an additional \Stars synthesized for the objective to reach some $\psure$-\textsf{good} SCC. However, in order to ensure that the resulting shielded policy is optimal in the above sense, we would need to enforce reaching the $\psure$-\textsf{good} SCC with the highest achievable reward. Determining this SCC would, however, need access to the reward structure of the MDP $M$ which we assume not to have. Hence, shielding \enquote{blindly} in this case would still be minimally interfering in the sense of \cref{thm:minimalinterference} but not necessarily optimal in the sense of \cref{cor:optimalshieldparity}. 
\end{remark}

\section{Proof for the Quality of Shielded Policies}\label{sec:naive}

Here we restate and prove the Theorem~5 of the main text, which shows that the frequency of visiting even color states can be increased by tuning the enforcement parameter $\gamma$.
\begin{restatable}{theorem}{restateFrequency}\label{thm:frequency}
    Given the premises of \cref{thm:optimalshield} with $\shield{\policy} = \policy|^{\template_\sure,\thres}_{\gamma}$ and $T$ being the set of even color states in the objective $\Phi$, for every frequency $0<\delta\leq \frac{1}{\abs{Q}}$, there exists parameters $\thres,\gamma > 0$ such that for every run $\run\sim\shield{\policy}$, it holds that $\freq(\run,T) \geq \delta$.
\end{restatable}
\begin{proof}
    Since $\Phi$ is a \buchi objective, by the construction of $\parityTemp_\sure$, as in \cite[Alg. 3]{AnandNS23}, $\template_\sure = (\emptyset,\emptyset,\livegroup = \{\livegroupSingle_i \mid i \in [1;n]\})$ such that $\win^{\sure}_\Phi  = Q$ can be 
    partitioned into $n$ groups $\{Q_i\}_{i=0}^n$ with $Q_0 = T$ and $\livegroupSingle_i$ being the edges from $Q_i$ to $Q_{i-1}$.
    Intuitively, the indices of the groups measure the closeness to the even color states $T$ and the live-groups are the edges that are used to ensure that a run progresses from higher index groups to lower index groups leading to the final group $T$.
    Furthermore, by \cref{claim:boundedcounter}, we know that each counter value of the live-groups is bounded by $K_{\thres,\gamma}$.
    As visiting $Q_i$ ensures visiting the source states of the live-group $\livegroupSingle_i$, it must hold that within $K_{\thres,\gamma}$ many visits to the group $Q_i$, the counter value of the live-group $\livegroupSingle_i$ is reset to zero, i.e., an edge from $\livegroupSingle_i$ is sampled at least once.
    Hence, applying the above argument recursively, we can show that within $K_{\thres,\gamma}^{n+1}$ many visits to the group $Q_n$ will ensure visiting the group $Q_{n}$ at least $K_{\thres,\gamma}^{n}$ times, which in turn ensures that the group $Q_{n-1}$ is visited at least $K_{\thres,\gamma}^{n-1}$ times and so on until the group $Q_0 = T$ is visited at least once.
    Finally, as every cycle of length $\abs{Q}$ visits some group at least twice, for every run $\run$ sampled from $\shield{\policy}$, it holds that
    \begin{align*}
        \freq(\run,T) &\geq \frac{1}{\abs{Q}\cdot K_{\thres,\gamma}^{n-1}}
         = \frac{1}{\abs{Q}}\cdot\min\left(1,\left(\frac{\gamma}{1/\thres-1}\right)^{n-1}\right).
    \end{align*}
    Hence, with appropriate values for $\thres$ and $\gamma$, we can ensure that the above expression is at least $\delta$.
\end{proof}

Note that \cref{thm:frequency} (Thm.~5 in the main text)  shows the existence of such parameters only for the case of surely satisfying $\Phi$. For the case of almost-sure satisfaction, the frequency would also depend on the transition probabilities of the MDP and hence, we cannot guarantee the existence of such parameters for every $\delta$, while the same intuition still holds.

\section{Large Figures from Introduction}

\begin{figure}[H]
    \centering
        \includegraphics[width=\textwidth]{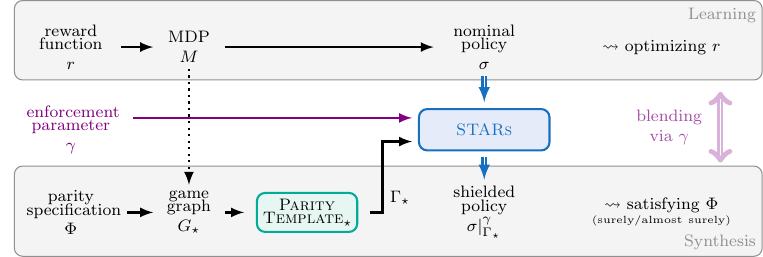}
        \includegraphics[width=0.6\textwidth]{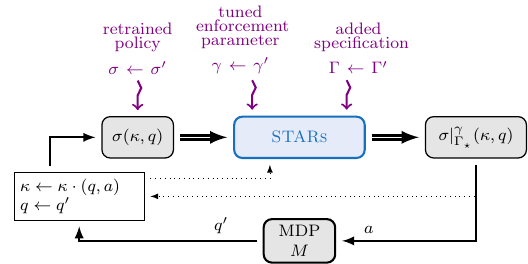} 
        \caption{Overview of \Stars synthesis (top) and runtime-application of \Stars (bottom). The detailed operation of \Stars is illustrated in \cref{fig:stesh}.
        Cyan components are taken from the literature. Purple components illustrate dynamic adaptability.
        }\label{fig:overview}
\end{figure}

\begin{figure}[H]
 \begin{center}
  \includegraphics[width=\textwidth]{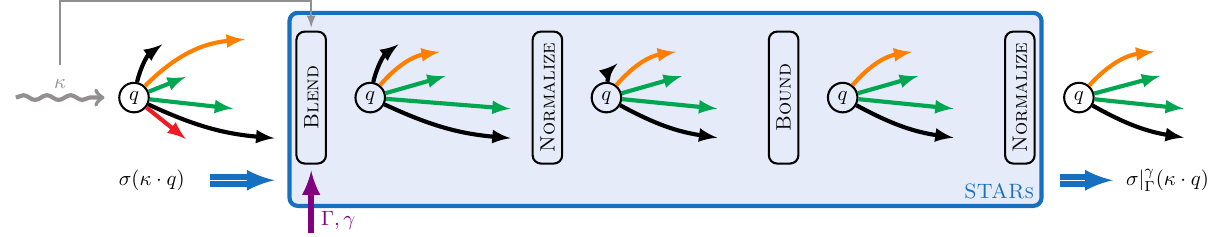}
  \caption{Illustration of dynamic interference via \Stars{}. Length of arrows indicate the relative probability of the corresponding action in $\mu\in\dist{A(q)}$. The strategy template $\Gamma=(\safegroup,\colivegroup,\livegroup)$ is illustrated via colors red ($\safegroup$), orange ($\colivegroup$) and green ($\livegroup$).}\label{fig:stesh}
 \end{center}
\end{figure}

\section{Additional Experimental Results}\label{sec:suppl-experiments}

\subsection{Larger Figures and More Screenshots from \benchmark{}}\label{sec:moreScreenshots}
Here, we provide larger and more detailed screenshots of the robot controlled by \tool{} in \benchmark{}.
\begin{itemize}
    \item Figure~4(a) of the main text is shown in \cref{fig:buchiRewardsComparison} with larger font size.
    \item Figure~4(b) of the main text is shown in \cref{fig:naiveVsStars} with larger font size.
    \item Figure~1(a) and 1(b) of the main text along with other screenshots are shown in \cref{figureAP:screenshots}.
\end{itemize}

\begin{figure}[H]
    \centering
        \includegraphics[width=0.6\textwidth]{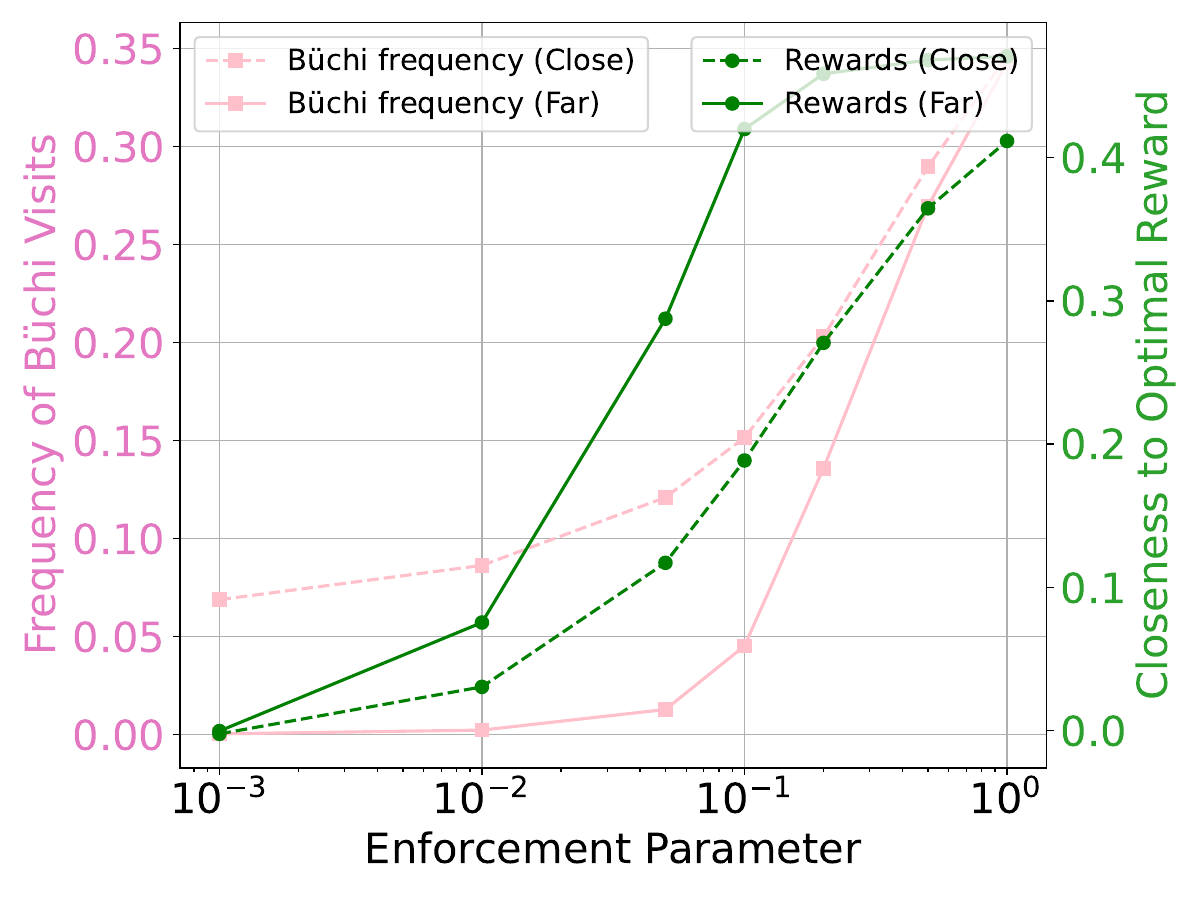}
        \caption{Effect of $\gamma$ on the \buchi frequency and average reward for \applyStars in \benchmark{}.}\label{fig:buchiRewardsComparison}
\end{figure}

\setcounter{figure}{4}
\begin{figure}[h]
    \centering
        \includegraphics[width=0.6\textwidth]{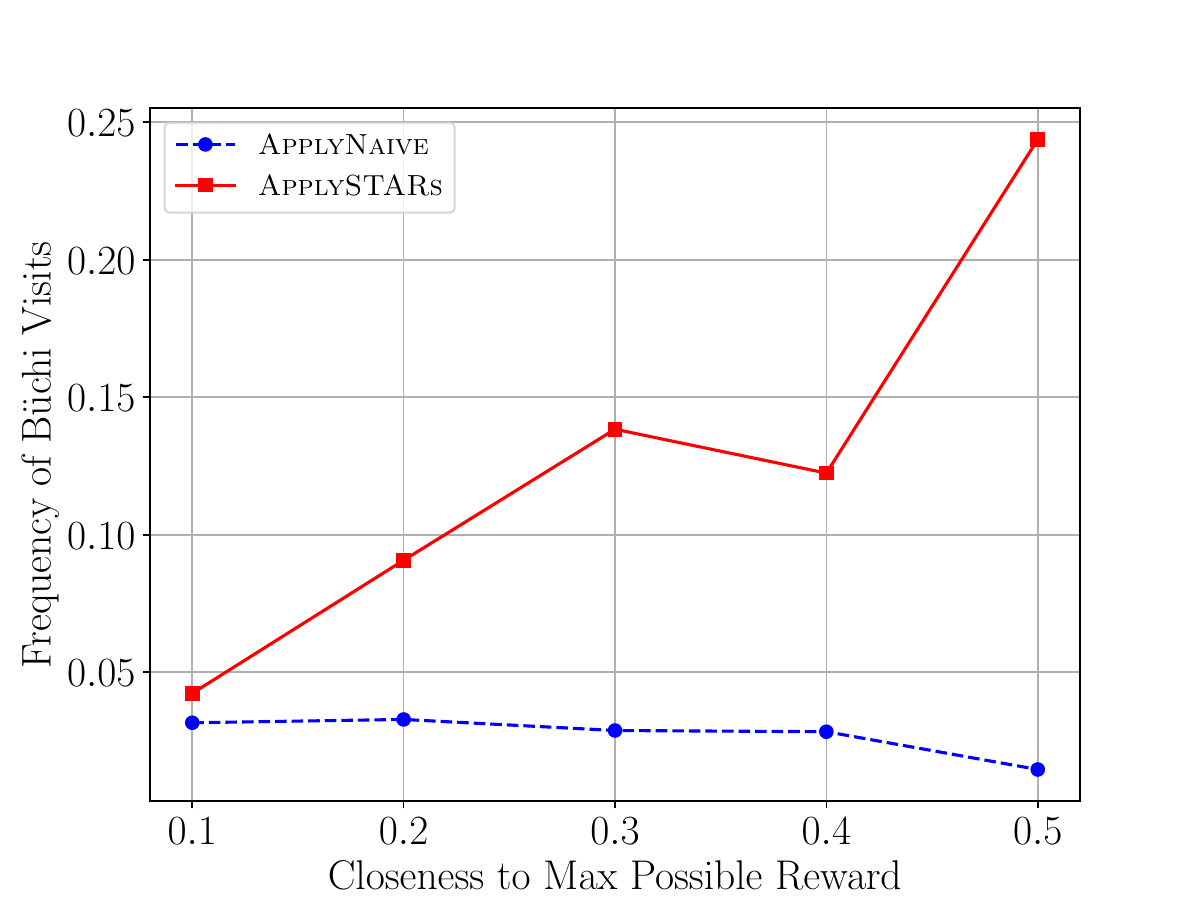}
        \caption{\buchi frequency vs. average reward for \applyStars and \applyNaive.}\label{fig:naiveVsStars}
\end{figure}

\begin{figure}[H]
	\centering
	\begin{subfigure}[b]{0.45\textwidth}
		\includegraphics[width=\textwidth]{./figures/noShield.png}
		\caption{Unshielded}
		\label{screenshotAP:unshielded}
	\end{subfigure}
	\quad
	\begin{subfigure}[b]{0.45\textwidth}
		\includegraphics[width=\textwidth]{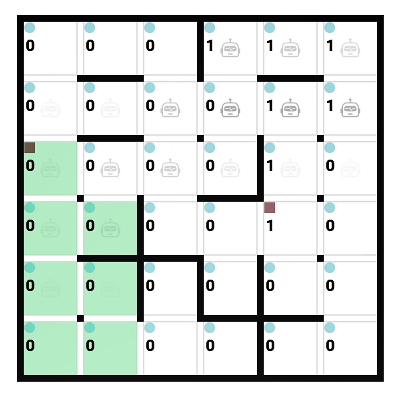}
		\caption{Low $\gamma$}
		\label{screenshotAP:lowShielded}
	\end{subfigure}
	\vspace{1em}
	\begin{subfigure}[b]{0.45\textwidth}
		\includegraphics[width=\textwidth]{./figures/highShield.png}
		\caption{High $\gamma$}
		\label{screenshotAP:highShielded}
	\end{subfigure}
	\quad
	\begin{subfigure}[b]{0.45\textwidth}
		\includegraphics[width=\textwidth]{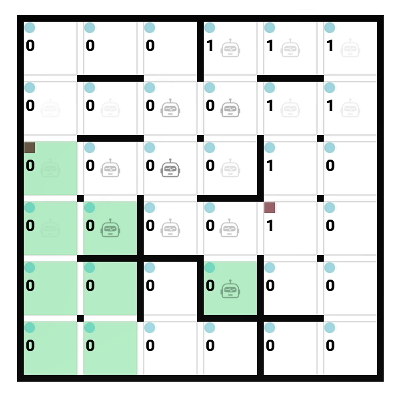}
		\caption{Online objective addition}
		\label{screenshotAP:onlineShielded}
	\end{subfigure}
	\caption{Screenshots of UI showing \tool controlled robot for an instance from \benchmark.}
	\label{figureAP:screenshots}
\end{figure}

\subsection{Additional Evaluation Details for LunarLander}
We show the additional evaluation details on how the enforcement parameters $\gamma$ affect the successful helipad landing rate and average number of steps to reach the helipad in \cref{fig:lunarLanderGamma}. In the main text, we only show the results for $\gamma=0.08$.

\begin{figure}[H]
    \centering
    \includegraphics[width=0.6\textwidth]{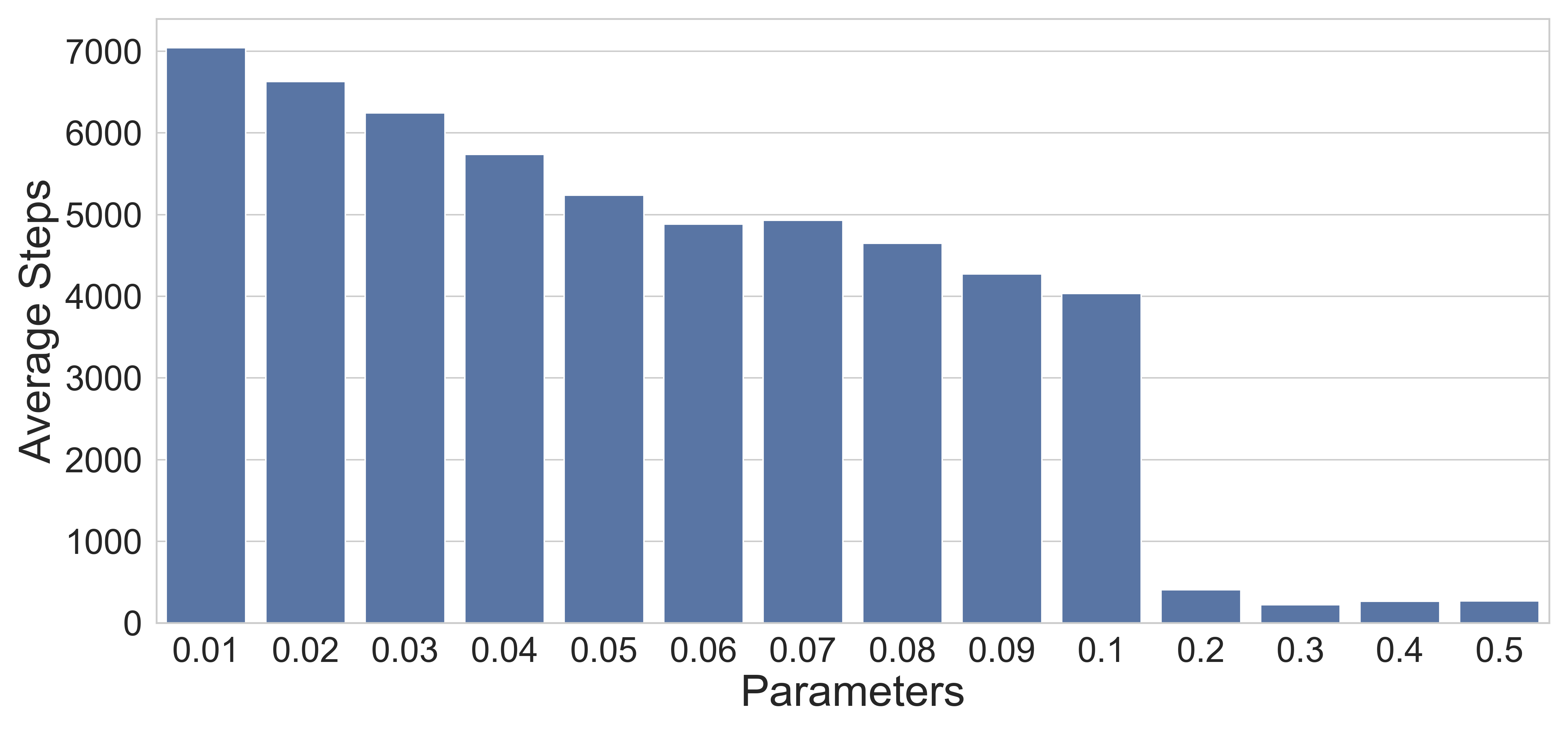}
    \includegraphics[width=0.6\textwidth]{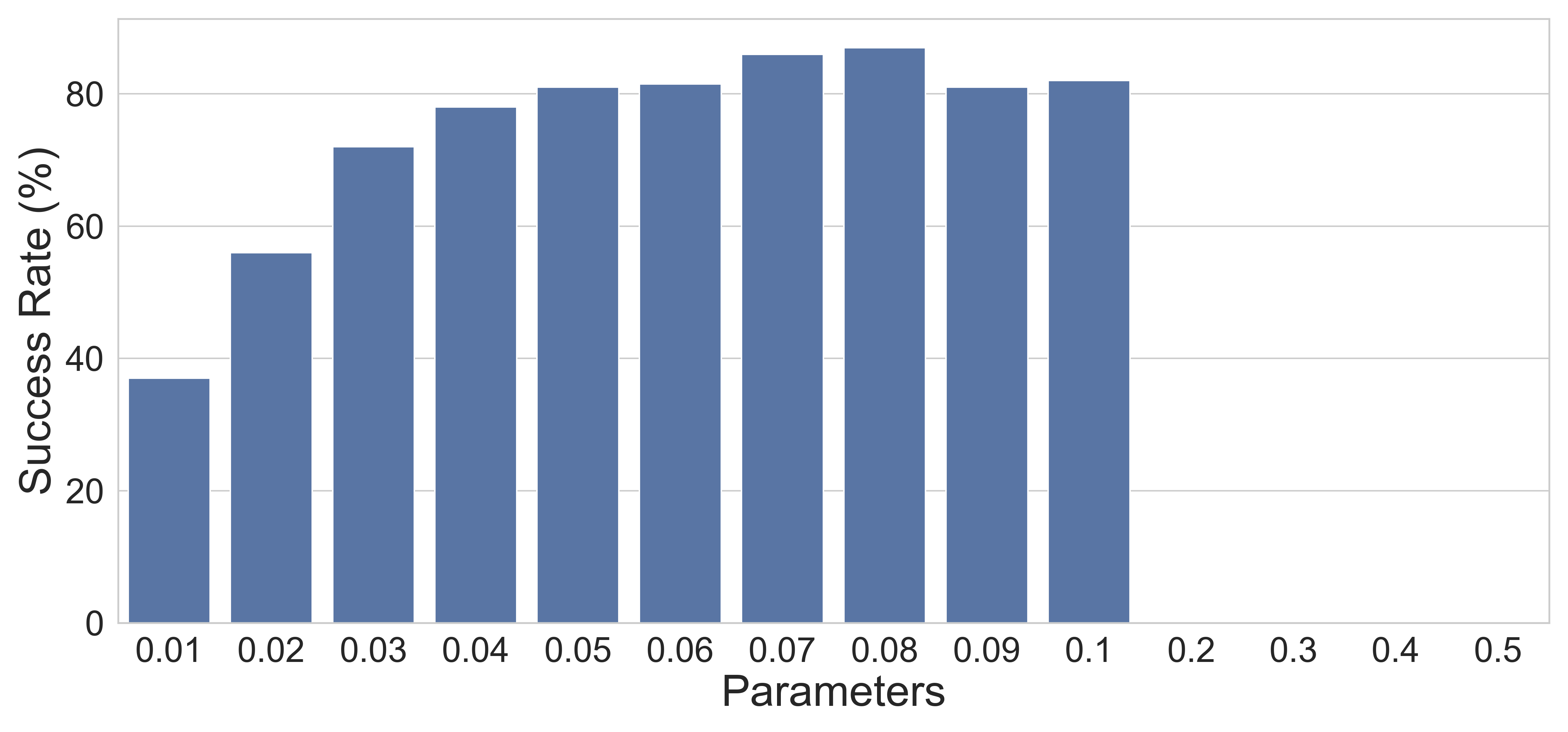}
    \caption{Effect of $\gamma$ on the successful helipad landing rate and average number of steps to reach the helipad in 200 environments of LunarLander.}\label{fig:lunarLanderGamma}
\end{figure}

\subsection{Larger Figures from Overcooked-AI}

\begin{figure}[H]
	\centering
	\includegraphics[width=0.6\textwidth]{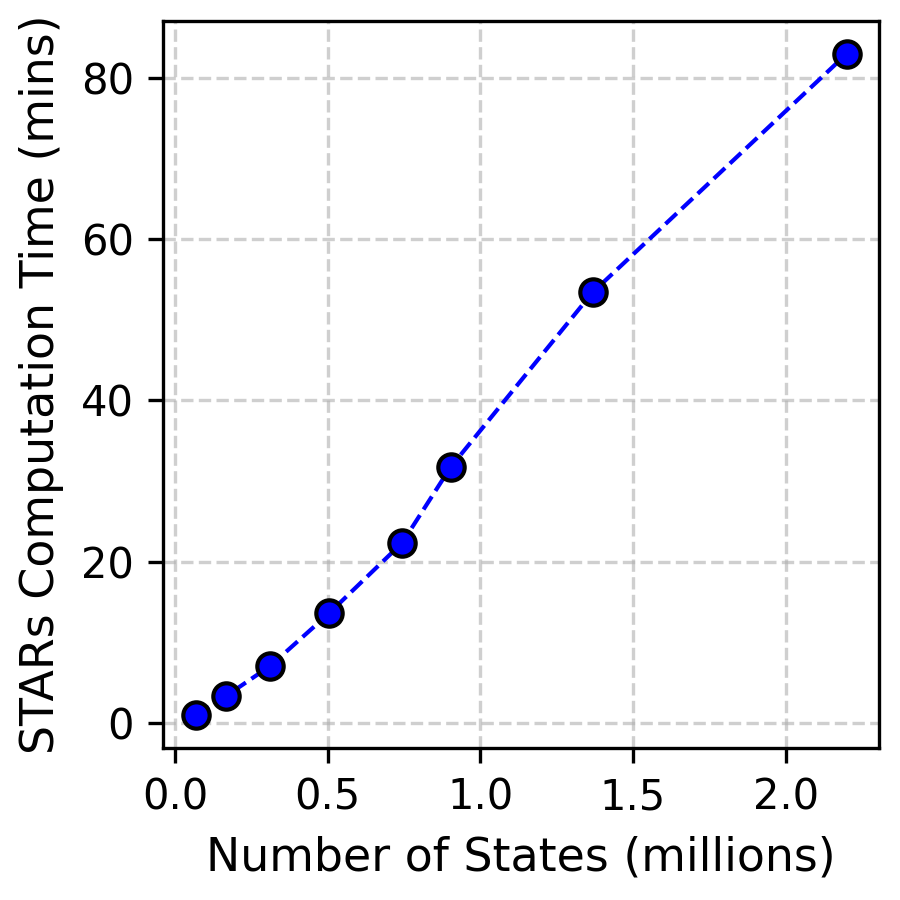}
	\caption{Scalability of \applyStars in the Overcooked-AI environment.}
	\label{fig:overcooked}
\end{figure}

\label{endofdocument}
\newoutputstream{pagestotal}
\openoutputfile{main.pgt}{pagestotal}
\addtostream{pagestotal}{\getpagerefnumber{endofdocument}}
\closeoutputstream{pagestotal}

\end{document}